\documentclass{article}

\PassOptionsToPackage{numbers}{natbib}
\usepackage[final]{neurips_2019}

\usepackage[utf8]{inputenc} % allow utf-8 input
\usepackage[T1]{fontenc}    % use 8-bit T1 fonts
\usepackage{hyperref}       % hyperlinks
\usepackage{url}            % simple URL typesetting
\usepackage{booktabs}       % professional-quality tables
\usepackage{amsfonts,amsmath,amssymb}       % blackboard math symbols
\usepackage{nicefrac}       % compact symbols for 1/2, etc.
\usepackage{microtype}      % microtypography
\usepackage{natbib}
\usepackage{graphicx} 
\usepackage{subcaption}  
\usepackage{amsthm} 
\usepackage{algorithm}  
\usepackage{algpseudocode}
\usepackage{color}
\usepackage{mathtools}
\usepackage{bbm}
\usepackage{wrapfig}
\usepackage{makecell}

\newtheorem{thm}{Theorem} 
\newtheorem{remark}{Remark} 
\newtheorem{lma}[thm]{Lemma}

\newtheorem{prop}[thm]{Proposition}

\usepackage{cleveref} % added by Wenbo: crefs
\crefname{thm}{theorem}{theorems}
\crefname{lma}{lemma}{lemmas}
\crefname{cor}{corollary}{corollaries}
\crefname{prop}{proposition}{propositions}

\newcommand{\bump}{\hspace{1.5em}}
\newcommand\opn[1]{\operatorname{#1}}

\newcommand\wt[1]{{\widetilde{#1}}}

\newcommand\ux[2]{{{#1}}^{#2}}
\newcommand{\lb}{\lbrack}
\newcommand{\rb}{\rbrack}

\newcommand{\RR}{\mathbb{R}}
\newcommand{\EX}{\mathbb{E}}

\DeclareMathOperator*{\argmin}{arg\,min}

\title{Leader Stochastic Gradient Descent for Distributed Training of Deep Learning Models: Extension}

\author{%
  $\text{Yunfei Teng}^{*,1}$\\
  \texttt{yt1208@nyu.edu} \\
  \And
  $\text{Wenbo Gao}^{*,2}$\\
  \texttt{wg2279@columbia.edu }\\
  \And Francois Chalus\\
  \texttt{chalusf3@gmail.com }\\
  \And 
  $\text{Anna Choromanska}^{**}$\\
  \texttt{ac5455@nyu.edu}\\
  \And
  Donald Goldfarb\\
  \texttt{goldfarb@columbia.edu }\\
  \And
  Adrian Weller\\
  \texttt{aw665@cam.ac.uk}
}

\begin{document}
\maketitle
{\let\thefootnote\relax\footnotetext{{*,1}: Equal contribution. Algorithm development and implementation on deep models.}}
{\let\thefootnote\relax\footnotetext{{*,2}: Equal contribution. Theoretical analysis and implementation on matrix completion.}}
{\let\thefootnote\relax\footnotetext{{* *}: Senior lead.}}
%\addtocounter{footnote}{-2}

\vspace{-0.2in}
\begin{abstract}
We consider distributed optimization under communication constraints for training deep learning models. We propose a new algorithm, whose parameter updates rely on two forces: a regular gradient step, and a corrective direction dictated by the currently best-performing worker (leader). Our method differs from the parameter-averaging scheme EASGD~\cite{EASGD} in a number of ways: (i) our objective formulation does not change the location of stationary points compared to the original optimization problem; (ii) we avoid convergence decelerations caused by pulling local workers descending to different local minima to each other (i.e. to the average of their parameters); (iii) our update by design breaks the curse of symmetry (the phenomenon of being trapped in poorly generalizing sub-optimal solutions in symmetric non-convex landscapes); and (iv) our approach is more communication efficient since it broadcasts only parameters of the leader rather than all workers. We provide theoretical analysis of the batch version of the proposed algorithm, which we call Leader Gradient Descent (LGD), and its stochastic variant (LSGD). Finally, we implement an asynchronous version of our algorithm and extend it to the multi-leader setting, where we form groups of workers, each represented by its own local leader (the best performer in a group), and update each worker with a corrective direction comprised of two attractive forces: one to the local, and one to the global leader (the best performer among all workers). The multi-leader setting is well-aligned with current hardware architecture, where local workers forming a group lie within a single computational node and different groups correspond to different nodes. For training convolutional neural networks, we empirically demonstrate that our approach compares favorably to  state-of-the-art baselines. \textbf{\textit{This work is a gentle extension of~\cite{DBLP:conf/nips/TengGCCGW19}}}. \textcolor{blue}{Finally, we developed a PyTorch-based comprehensive distributed training library for deep networks that can be found in~\url{https://github.com/yunfei-teng/LSGD}. The library contains several methods, among them LSGD.}
\end{abstract} \vspace{-0.2in}

\section{Introduction}
\label{sec:Intro}

As deep learning models and data sets grow in size, it becomes increasingly helpful to parallelize their training over a distributed computational environment. These models lie at the core of many modern machine-learning-based systems for image recognition~\cite{NIPS2012_4824}, speech recognition~\cite{DBLP:conf/icassp/Abdel-HamidMJP12}, natural language processing~\cite{DBLP:conf/emnlp/WestonCA14}, and more. This paper focuses on the parallelization of the data, not the model, and considers collective communication scheme~\cite{DBLP:journals/corr/WickramasingheL16} that is most commonly used nowadays. A typical approach to data parallelization in deep learning~\cite{B-NH2018arxiv,GAJKB2018SPAA} uses multiple workers that run variants of SGD~\cite{bottou-98x} on different data batches. Therefore, the effective batch size is increased by the number of workers. Communication ensures that all models are synchronized and critically relies on a scheme where each worker broadcasts its parameter gradients to all the remaining workers. This is the case for DOWNPOUR~\cite{DOWNPOUR} (its decentralized extension, with no central parameter server, based on the ring topology can be found in~\cite{pmlr-v80-lian18a}) or Horovod~\cite{sergeev2018horovod} methods. These techniques require frequent communication (after processing each batch) to avoid instability/divergence, and hence are communication expensive. Moreover, training with a large batch size usually hurts generalization~\cite{KMNST2017ICLR,JKABFBS2018ICLRW,SL2018ICLR} and convergence speed~\cite{MBB2018ICML,DBLP:journals/corr/abs-1708-03888}. LARS~\cite{LARS} proposes a layer-wise learning rate scaling method to overcome the difficulties of large batch training, but this method does not guarantee accelerations of convergence when the batch size is not large enough.

Another approach, called Elastic Averaging (Stochastic) Gradient Decent, EA(S)GD~\cite{EASGD}, introduces elastic forces linking the parameters of the local workers with central parameters computed as a moving average over time and space (i.e. over the parameters computed by local workers). This method allows less frequent communication as workers by design do not need to have the same parameters but are instead periodically pulled towards each other. The objective function of EASGD, however, has stationary points which are not stationary points of the underlying objective function (see Proposition~\ref{prop:EASGD} in the Supplement), thus optimizing it may lead to sub-optimal solutions for the original problem. Further, EASGD can be viewed as a parallel extension of the averaging SGD scheme~\cite{polyak1992acceleration} and as such it inherits the downsides of the averaging policy. On non-convex problems, when the iterates are converging to different local minima (that may potentially be globally optimal), the averaging term can drag the iterates in the wrong directions and significantly hurt the convergence speed of both local workers and the master. In symmetric regions of the optimization landscape, the elastic forces related with different workers may cancel each other out causing the master to be permanently stuck in between or at the maximum between different minima, and local workers to be stuck at the local minima or on the slopes above them. This can result in arbitrarily bad generalization error. We refer to this phenomenon as the ``curse of symmetry''. Landscape symmetries are common in a plethora of non-convex problems~\cite{LLAHLWZ2019IEEETRINFTH,ge2017no,DBLP:journals/focm/SunQW18,DBLP:journals/tit/SunQW17,ge2016matrix}, including deep learning~\cite{DBLP:journals/corr/BadrinarayananM15, DBLP:conf/aistats/ChoromanskaHMAL15,pmlr-v80-liang18a,K2016NIPS}.\\ 
\begin{wrapfigure}{r}{0.55\textwidth} 
\vspace{-0.18in}
  \begin{center}
    \includegraphics[width=0.55\textwidth]{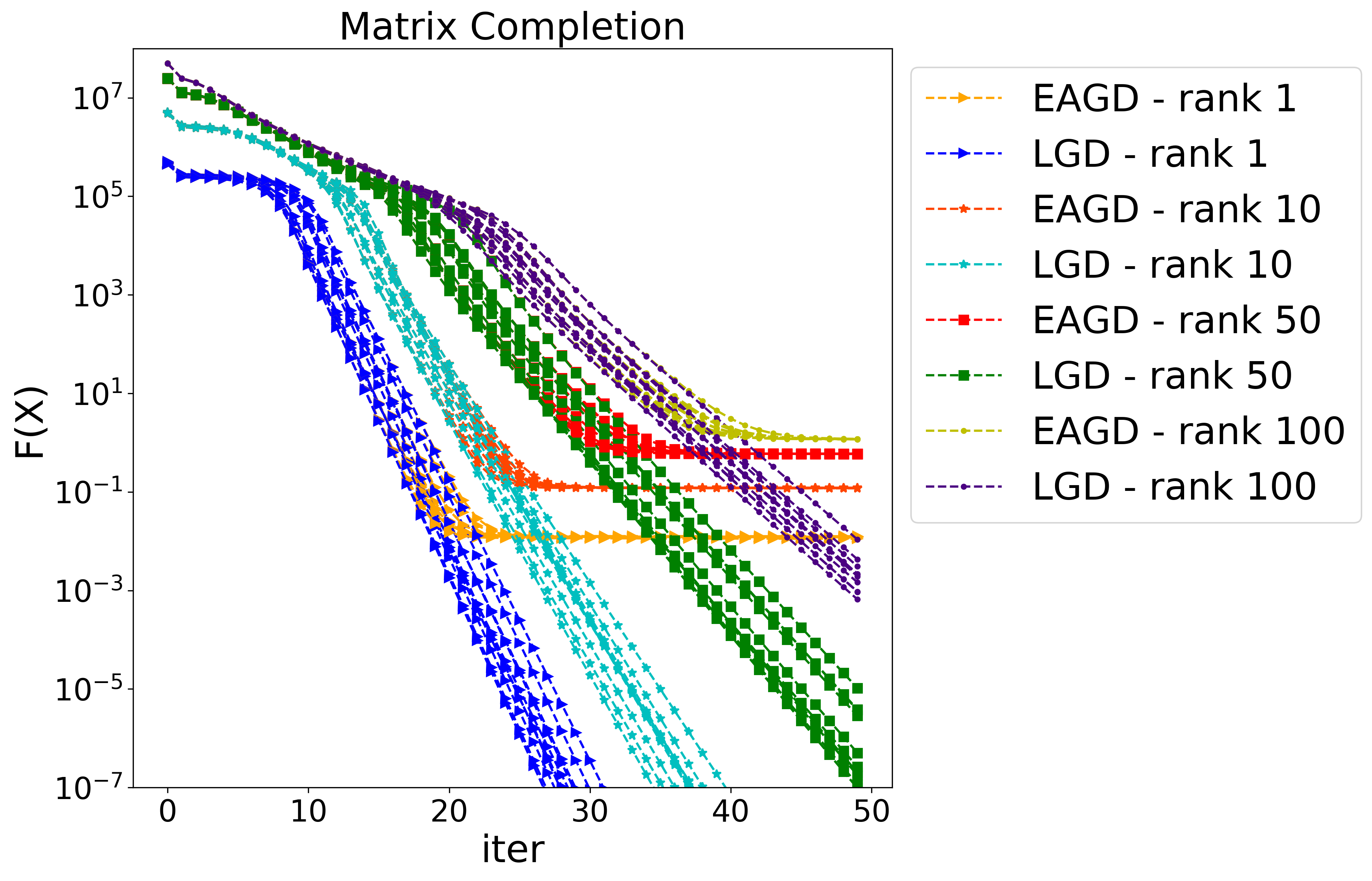}
  \end{center}
  \vspace{-0.19in}
  \caption{Low-rank matrix completion problems solved with EAGD and LGD. The dimension $d = 1000$ and four ranks $r \in \{1,10,50,100\}$ are used. The reported value for each algorithm is the value of the best worker ($8$ workers are used in total) at each step.}
  \label{fig:matrix_completion}
  \vspace{-0.15in}
\end{wrapfigure}
This paper revisits the EASGD update and modifies it in a simple, yet powerful way which overcomes the above mentioned shortcomings of the original technique. We propose to replace the elastic force relying on the average of the parameters of local workers by an attractive force linking the local workers and the current best performer among them (leader). Our approach reduces the communication overhead related with broadcasting parameters of all workers to each other, and instead requires broadcasting only the leader parameters. The proposed approach easily adapts to a typical hardware architecture comprising of multiple compute nodes where each node contains a group of workers and local communication, within a node, is significantly faster than communication between the nodes. We propose a multi-leader extension of our approach that adapts well to this hardware architecture and relies on forming groups of workers (one per compute node) which are attracted both to their local and global leader. To reduce the communication overhead, the correction force related with the global leader is applied less frequently than the one related with the local leader.

Finally, our L(S)GD approach, similarly to EA(S)GD, tends to explore wide valleys in the optimization landscape when the pulling force between workers and leaders is set to be small. This property often leads to improved generalization performance of the optimizer~\cite{DBLP:journals/corr/ChaudhariCSL16,DBLP:journals/corr/ChaudhariBZST17}.

The paper is organized as follows: Section~\ref{sec:Alg} introduces the L(S)GD approach, Section~\ref{sec:TA} provides theoretical analysis, Section~\ref{sec:Exp} contains empirical evaluation, and finally Section~\ref{sec:Con} concludes the paper. Theoretical proofs and additional theoretical and empirical results are contained in the Supplement. 

\section{Leader (Stochastic) Gradient Descent ``L(S)GD'' Algorithm}
\label{sec:Alg}

\subsection{Motivating example}

\Cref{fig:matrix_completion} illustrates how elastic averaging can impair convergence. To obtain the figure we applied EAGD (Elastic Averaging Gradient Decent) and LGD to the matrix completion problem of the form: $\min_X \left\{ \frac{1}{4}\|M - XX^T\|_F^2: X \in \RR^{d \times r} \right\}$. This problem is non-convex but is known to have the property that all local minimizers are global minimizers \cite{LLAHLWZ2019IEEETRINFTH}. For four choices of the rank $r$, we generated $10$ random instances of the matrix completion problem, and solved each with EAGD and LGD, initialized from the same starting points (we use $8$ workers). For each algorithm, we report the progress of the \emph{best} objective value at each iteration, over all workers. \Cref{fig:matrix_completion} shows the results across $10$ random experiments for each rank.

It is clear that EAGD slows down significantly as it approaches a minimizer. Typically, the center $\wt{X}$ of EAGD is close to the average of the workers, which is a poor solution for the matrix completion problem when the workers are approaching different local minimizers, even though all local minimizers are globally optimal. This induces a pull on each node \emph{away} from the minimizers, which makes it extremely difficult for EAGD to attain a solution of high accuracy. In comparison, LGD does not have this issue. Further details of this experiment, and other illustrative examples of the difference between EAGD and LGD, can be found in the Supplement. %Section~\ref{sec:Il} (Supplement) also contains another illustrative example of the difference between EAGD and LGD.

\subsection{Symmetry-breaking updates}

Next we explain the basic update of the L(S)GD algorithm. Consider first the single-leader setting and the problem of minimizing loss function $L$ in a parallel computing environment. The optimization problem is given as

\vspace{-0.25in}
\begin{equation}
    \min_{x^1, x^2, \dots, x^l}L(x^1, x^2, \dots, x^l) \coloneqq 
    \min_{x^1, x^2, \dots, x^l}\sum_{i=1}^l\mathbb{E}[f(x^{i};\xi^{i})]
    + \frac{\lambda}{2} \lvert\lvert x^{i} - \Tilde{x} \rvert\rvert^2,
    \label{eq:objective1}
\end{equation}
\vspace{-0.15in}

where $l$ is the number of workers, $x^1,x^2,\dots,x^l$ are the parameters of the workers and $\tilde{x}$ are the parameters of the leader. The best performing worker, i.e. $\Tilde{x}=\argmin\limits_{x^{1}, x^2, \dots, x^{l}}
    \mathbb{E}[f(x^{i};\xi^{i})]$), and $\xi^i$s are data samples drawn from some probability distribution $\mathcal{P}$. $\lambda$ is the hyperparameter that denotes the strength of the force pulling the workers to the leader. In the theoretical section we will refer to $\EX[f(x^{i};\xi^{i})]$ as simply $f(x^{i})$. This formulation can be further extended to the multi-leader setting. The optimization problem is modified to the following form

\vspace{-0.2in}
\begin{eqnarray}
    &&\min_{x^{1,1}, x^{1,2}, \dots, x^{n,l}}L(x^{1,1}, x^{1,2}, \dots, x^{n,l}) \nonumber\\ 
    &&\coloneqq \min_{x^{1,1}, x^{1,2}, ..., x^{n, l}} \sum_{j=1}^{n} \sum_{i=1}^{l}
    \mathbb{E}[f(x^{j, i};\xi^{j, i})]
    + \frac{\lambda}{2} \lvert\lvert x^{j, i} - \Tilde{x}^j \rvert\rvert^2 
    + \frac{\lambda_G}{2} \lvert\lvert x^{j, i} - \Tilde{x} \rvert\rvert^2,
    \label{eq:objective2}
\end{eqnarray}
\vspace{-0.15in}

where $n$ is the number of groups, $l$ is the number of workers in each group, $\tilde{x}^j$ is the local leader of the $j^{\text{th}}$ group (i.e. $\Tilde{x}^j=\argmin_{x^{j,1}, x^{j,2}, \dots, x^{j,l}}
    \mathbb{E}[f(x^{j,i};\xi^{j,i})]$), $\tilde{x}$ is the global leader (the best worker among local leaders, i.e. $\Tilde{x} = \argmin\limits_{x^{1,1}, x^{1,2}, \dots, x^{n,l}}
    \mathbb{E}[f(x^{j,i};\xi^{j,i})]$), $x^{j,1}, x^{j,2}, \dots, x^{j,l}$ are the parameters of the workers in the $j^{\text{th}}$ group, and $\xi^{j,i}$s are the data samples drawn from $\mathcal{P}$. $\lambda$ and $\lambda_G$ are the hyperparameters that denote the strength of the forces pulling the workers to their local and global leader respectively. 

The updates of the LSGD algorithm are captured below, where $t$ denotes iteration. The first update shown in Equation~\ref{eq:one} is obtained by taking the gradient descent step on the objective in Equation \ref{eq:objective2} with respect to variables $x^{j, i}$. The stochastic gradient of $\mathbb{E}[f(x^{i};\xi^{i})]$ with respect to $x^{j, i}$ is denoted as $g_t^{j, i}$ (in case of LGD the gradient is computed over all training examples) and $\eta$ is the learning rate.

\vspace{-0.1in}
\begin{equation}
    x_{t+1}^{j, i}=x_{t}^{j, i}-\eta g_t^{j, i}(x_t^{j, i})
    - \lambda (x_t^{j, i} - \Tilde{x}_t^j)
    - \lambda_G (x_t^{j, i} - \Tilde{x}_t)
    \label{eq:one}
\end{equation}
where $\Tilde{x}_{t+1}^j$ and $\Tilde{x}_{t+1}$ are the local and global leaders defined above.

%\begin{equation}
%    \Tilde{x}_{t+1}^j=\argmin_{x^{j, 1}, ..., x^{j, n}}
%    \mathbb{E}[f(x_t^{j, i};\xi_t^{j, i})]
%    \label{eq:bestlocal}
%\end{equation}

%\begin{equation}
%    \Tilde{x}_{t+1}=\argmin_{x^{1, 1}, ..., x^{m, n}}
%    \mathbb{E}[f(x_t^{j, i};\xi_t^{j, i})]
%    \label{eq:bestgloabl}
%\end{equation}

\begin{algorithm}[t!] %Complete LSGD algorithm
\caption{LSGD Algorithm (Asynchronous)} % \caption should be before \label
\label{alg:LSGD}
\begin{algorithmic}
\State \textbf{Input:} pulling coefficients $\lambda, \lambda_G$, searching scopes $\gamma, \gamma_G$, learning rate $\eta$, local/global communication periods $\tau, \tau_G$
%\textbf{Input:} $\lambda, \lambda_G$: pulling %coefficients; $\eta$: learning rate; $\tau, \tau_G$: local %and global communication periods\\
%\hspace{0.41in}// $\xi_{t}^{j, i}$ denotes data sample(s) %received by the $i^{\text{th}}$ worker in the %$j^{\text{th}}$ group at iteration $t$
\State \textbf{Initialize:} \\
\hspace{0.41in}Randomly initialize $x^{1, 1}, x^{1, 2}, ..., x^{n, l}$\\
\hspace{0.41in}Set iteration counters $t^{j,i} = 0$\\
\hspace{0.41in}Set $\Tilde{x}^j_0=\argmin\limits_{x^{j, 1}, ..., x^{j, l}} 
     \mathbb{E}[f(x^{j, i};\xi_{0}^{j, i})], \Tilde{x}_0=\argmin\limits_{x^{1, 1}, ..., x^{n, l}} 
     \mathbb{E}[f(x^{j, i};\xi_{0}^{j, i})]$;
%\hspace{0.41in}Set $\Tilde{x}_0=\argmin_{x^{1, 1}, ..., x^{n, l}} 
%     \mathbb{E}[f(x^{j, i};\xi_{0}^{j, i})]$.
\Repeat
    \ForAll{$j = 1,2,\dots,n$, $i = 1,2,\dots,l$} 
    \Comment{Do in parallel for each worker}
    \State Draw random sample $\xi^{j,i}_{t^{j,i}}$
    \State $x^{j,i} \xleftarrow{} x^{j,i} - \eta g_t^{j,i}(x^{j,i}) -\frac{\gamma}{\tau} (x^{j, i} - \Tilde{x}^j) - \frac{\gamma_G}{\tau_G} (x^{j, i} - \Tilde{x})$
    \State $t^{j,i} = t^{j,i} + 1$;
    \If{$nl\tau$ divides $(\sum\limits_{j=1}^{n}\sum\limits_{i=1}^{l} t^{j,i})$}
        \State $\Tilde{x}^j = \argmin_{x^{j, 1}, ..., x^{j, l}} 
            \mathbb{E}[f(x^{j, i};\xi^{j, i}_{t^{j,i}})]$.
        \Comment{Determine the local best workers}
        \State $x^{j, i}\xleftarrow{}x^{j, i}
               - \lambda (x^{j, i} - \Tilde{x}^j)$
        \Comment{Pull to the local best workers}
    \EndIf
    \If{$nl\tau_G$ divides  $(\sum\limits_{j=1}^{n}\sum\limits_{i=1}^{l} t^{j,i})$}
        \State $\Tilde{x}=\argmin_{x^{1, 1}, ..., x^{n, l}} 
            \mathbb{E}[f(x^{j, i};\xi^{j, i}_{t^{j,i}})]$.
        \Comment{Determine the global best worker}
        \State $x^{j, i}\xleftarrow{}x^{j, i}
               - \lambda_G (x^{j, i} - \Tilde{x})$
        \Comment{Pull to the global best worker}
    \EndIf
    \EndFor
\Until{termination}
\end{algorithmic}
\end{algorithm}
\setlength{\textfloatsep}{0.3cm}
\setlength{\floatsep}{0.3cm}

Equation~\ref{eq:one} describes the update of any given worker and is comprised of the regular gradient step and two corrective forces (in single-leader setting the third term disappears as $\lambda_G = 0$ then). These forces constitute the communication mechanism among the workers and pull all the workers towards the currently best local and global solution to ensure fast convergence. As opposed to EASGD, the updates performed by workers in LSGD break the curse of symmetry and avoid convergence decelerations that result from workers being pulled towards the average which is inherently influenced by poorly performing workers. In this paper, instead of pulling workers to their averaged parameters, we propose the mechanism of pulling the workers towards the leaders. The flavor of the update resembles a particle swarm optimization approach~\cite{488968}, which is not typically used in the context of stochastic gradient optimization for deep learning. Our method may therefore be viewed as a dedicated particle swarm optimization approach for training deep learning models in the stochastic setting and parallel computing environment.

Next we describe the LSGD algorithm in more detail. We rely on the collective communication scheme. In order to reduce the amount of communication between workers, it is desired to update the leaders less frequently than every iteration.
% \textcolor{green}{$\gamma$ is nowhere defined and it is not clear what it governs.}
Also, in practice each worker can have a different speed. To prevent waiting for the slower workers and achieve communication efficiency, we implement the algorithm in the asynchronous operation mode. In this case, the communication period is determined based on the total number of iterations computed across all workers and the communication is performed every $nl\tau$ or $nl\tau_G$ iterations, where $\tau$ and $\tau_G$ denote local and global communication periods, respectively. In practice, we use $\tau_G > \tau$ since communication between workers lying in different groups is more expensive than between workers within one group, as explained above. When communication occurs, all workers are updated at the same time (i.e. pulled towards the leaders) in order to take advantage of the collective communication scheme. Between communications, workers run their own local SGD optimizer to find a better solution nearby the leaders.
The hyperparameters $\gamma$ and $\gamma_G$ control the overall searching scopes of the workers around their local leaders and global leader, respectively. In practice, to encourage the workers to stay inside their search ranges, we pull the workers towards the leaders as well.
The resulting LSGD method is very simple, and is depicted in Algorithm~\ref{alg:LSGD}. \textbf\textit{{Note that this algorithm is a slight modification of the algorithm from~\cite{DBLP:conf/nips/TengGCCGW19}}}. The modification is based on the observation that pulling the workers towards leaders could improve the performance even if the leaders become stale.

The next section provides a theoretical description of the single-leader batch (LGD) and stochastic (LSGD) variants of our approach.

\section{Theoretical Analysis}\label{sec:TA}

We assume without loss of generality that there is a single leader. The objective function with multiple leaders is given by $f(x) + \frac{\lambda_1}{2}\|x - z_1\|^2 + \ldots + \frac{\lambda_c}{2}\|x - z_c\|^2$, which is equivalent to $f(x) + \frac{\Lambda}{2}\|x - \wt{z}\|^2$ for $\Lambda = \sum_{i=1}^c \lambda_i$ and $\wt{z} = \frac{1}{\Lambda}\sum_{i=1}^c \lambda_i z_i$. Proofs for this section are deferred to the Supplement.

\subsection{Convergence Rates for Stochastic Strongly Convex Optimization}\label{sub:strongconvex}

We first show that LSGD obtains the same convergence rate as SGD for stochastic strongly convex problems \cite{BCN2018SIAMREV}. In \Cref{sub:improve} we discuss how and when LGD can obtain \emph{better} search directions than gradient descent. We discuss non-convex optimization in \Cref{sub:nonconvex}. Throughout \Cref{sub:strongconvex}, $f$ will typically satisfy:

\textbf{Assumption 1} $f$ is $M$-Lipschitz-differentiable and $m$-strongly convex, which is to say, the gradient $\nabla f$ satisfies $\|\nabla f(x) - \nabla f(y)\| \leq M\|x - y\|$, and $f$ satisfies $f(y) \geq f(x) + \nabla f(x)^T(y-x) + \frac{m}{2}\|y-x\|^2$.
We write $x^\ast$ for the unique minimizer of $f$, and $\kappa := \frac{M}{m}$ for the condition number of $f$.

\subsubsection{Convergence Rates}\label{subsub:rates}
The key technical result is that LSGD satisfies a similar one-step descent in expectation as SGD, with an additional term corresponding to the pull of the leader. To provide a unified analysis of `pure' LSGD as well as more practical variants where the leader is updated infrequently or with errors, we consider a general iteration $x_+ = x - \eta(\wt{g}(x) + \lambda(x-z))$, where $z$ is an arbitrary guiding point; that is, $z$ may not be the minimizer of $\ux{x}{1},\ldots,\ux{x}{p}$, nor even satisfy $f(z) \leq f(\ux{x}{i})$. Since the nodes operate independently except when updating $z$, we may analyze LSGD steps for each node individually, and we write $x = \ux{x}{i}$ for brevity.

\begin{thm}\label{standard_rate}
Let $f$ satisfy Assumption 1. Let $\wt{g}(x)$ be an unbiased estimator for $\nabla f(x)$ with $\opn{Var}(\wt{g}(x)) \leq \sigma^2 + \nu \|\nabla f(x)\|^2$, and let $z$ be any point. Suppose that $\eta,\lambda$ satisfy $\eta \leq (2M(\nu+1))^{-1}$ and $\eta\lambda \leq (2\kappa)^{-1}, \eta\sqrt{\lambda} \leq (\kappa\sqrt{2m})^{-1}$. Then the LSGD step satisfies
\begin{equation}\label{eq:lsgd_onestep_descent}
\EX f(x_+) - f(x^\ast) \leq (1-m\eta)(f(x) - f(x^\ast)) - \eta\lambda( f(x) - f(z)) + \frac{\eta^2 M}{2}\sigma^2.
\end{equation}
Note the presence of the new term $-\eta\lambda(f(x)-f(z))$ which speeds up convergence when $f(z) \leq f(x)$, i.e the leader is better than $x$. If the leader $z_k$ is always chosen so that $f(z_k) \leq f(x_k)$ at every step $k$, then $\limsup_{k \rightarrow \infty} \EX f(x_k) - f(x^\ast) \leq \frac{1}{2}\eta\kappa\sigma^2$. If $\eta$ decreases at the rate $\eta_k = \Theta(\frac{1}{k})$, then $\EX f(x_k) - f(x^\ast) \leq O(\frac{1}{k})$.
\end{thm}

The $O(\frac{1}{k})$ rate of LSGD matches that of comparable distributed methods. Both Hogwild~\cite{RRWN2011NIPS} and EASGD achieve a rate of $O(\frac{1}{k})$ on strongly convex objective functions. We note that published convergence rates are not available for many distributed algorithms (including DOWNPOUR~\cite{DOWNPOUR}).

%\textcolor{red}{\textbf{Remark 1. Comparison with other methods.} For the DOWNPOUR method specifically, the original paper is purely empirical, and we are not aware of any published convergence analysis for it. The closest proxy for DOWNPOUR with a known rate is the Hogwild method of \cite{RRWN2011NIPS}, which achieves a rate of $O(1/k)$ (up to a logarithmic term). The EASGD method also achieves $O(1/k)$ on strongly convex objective functions. In both cases, this rate matches that of LSGD. Unfortunately many distributed algorithms are presented without theoretical analysis (e.g. DOWNPOUR, PARLE).}

\subsubsection{Communication Periods}\label{subsub:comm}
In practice, communication between distributed machines is costly. The LSGD algorithm has a \emph{communication period} $\tau$ for which the leader is only updated every $\tau$ iterations, so each node can run independently during that period. This $\tau$ is allowed to differ between nodes, and over time, which captures the asynchronous and multi-leader variants of LSGD. We write $x_{k,j}$ for the $j$-th step during the $k$-th period. It may occur that $f(z) > f(x_{k,j})$ for some $k,j$, that is, the current solution $x_{k,j}$ is now better than the last selected leader. In this case, the leader term $\lambda(x-z)$ may no longer be beneficial, and instead simply pulls $x$ toward $z$. There is no general way to determine how many steps are taken before this event. However, we can show that if $f(z) \geq f(x)$, then 
\begin{equation}\label{eq:stale_leader}
\EX f(x_+) \leq f(z) + \frac{1}{2}\eta^2 M \sigma^2,
\end{equation}
so the solution will not become \emph{worse} than a stale leader (up to gradient noise). As $\tau$ goes to infinity, LSGD converges to the minimizer of $\psi(x) = f(x) + \frac{\lambda}{2}\|x - z\|^2$, which is quantifiably better than $z$ as captured in \Cref{w_value}. Together, these facts show that LSGD is safe to use with long communication periods as long as the original leader is good.

\begin{thm}\label{w_value}
Let $f$ be $m$-strongly convex, and let $x^\ast$ be the minimizer of $f$. For fixed $\lambda,z$, define $\psi(x) = f(x) + \frac{\lambda}{2}\|x - z\|^2$. The minimizer $w$ of $\psi$ satisfies $
f(w) - f(x^\ast) \leq \frac{\lambda}{m+\lambda}(f(z) - f(x^\ast))$.
\end{thm}

The theoretical results here and in \Cref{subsub:rates} address two fundamental instances of the LSGD algorithm: the `synchronous' case where communication occurs each round, and the `infinitely asynchronous' case where communication periods are arbitrarily long. For unknown periods $\tau > 1$, it is difficult to demonstrate general quantifiable improvements beyond (\ref{eq:stale_leader}), but we note that (\ref{eq:lsgd_onestep_descent}), \Cref{w_value}, and the results on stochastic leader selection (\Cref{subsub:stochleader,sub:proofs_stochleader}) can be combined to analyze specific instances of the asynchronous LSGD.

% In our experiments, we employ another method to avoid the issue of stale leaders. To ensure that the leader is good, we perform an LSGD step only on the first step after a leader update, and then take standard SGD steps for the remainder of the communication period. 
In our experiments, we employ another method to alleviate the issue of stale leaders. To increase the importance of the leader when it is good, we perform a larger LSGD step on the first step after a leader update, and take smaller LSGD steps for the remainder of the communication period ($\frac{\gamma}{\tau} < \lambda$ and $\frac{\gamma_G}{\tau_G} < \lambda_G$).

%\textcolor{red}{\textbf{Remark 2.} Regarding the asynchronous algorithm, we opted to present results for two settings, one-step round and `arbitrarily long' round, in the main paper because numerous variations of communication schedules are possible. The behavior of the algorithm given an unknown round length~>1 is very difficult to measure in a useful way (i.e., to find quantifiable improvements) since it requires estimating the lowest value obtained along the trajectory when the leader is kept fixed, which makes it difficult to define a rate for the `general' asynchronous method. Note that several useful lemmas for the analysis of stochastic leader selection are presented in the Supplement, which can also be combined to analyze combinations of the asynchronous algorithm with stochastic leader selection.} 

\subsubsection{Stochastic Leader Selection}\label{subsub:stochleader}
Next, we consider the impact of selecting the leader with errors. In practice, it is often costly to evaluate $f(x)$, as in deep learning. Instead, we estimate the values $f(\ux{x}{i})$, and then select $z$ as the variable having the smallest estimate. Formally, suppose that we have an unbiased estimator $\wt{f}(x)$ of $f(x)$, with uniformly bounded variance. At each step, a single sample $y_1,\ldots,y_p$ is drawn from each estimator $\wt{f}(\ux{x}{1}), \ldots, \wt{f}(\ux{x}{p})$, and then $z = \{\ux{x}{i}: y_i = \min \{y_1,\ldots,y_p\}\}$. We refer to this as \emph{stochastic leader selection}. The stochastic leader satisfies $\EX f(z) \leq f(z_{true}) + 4\sqrt{p}\sigma_f$, where $z_{true}$ is the true leader (see supplementary materials). Thus, the error introduced by the stochastic leader contributes an additive error of at most $4\eta \lambda \sqrt{p}\sigma_f$. Since this is of order $\eta$ rather than $\eta^2$, we cannot guarantee convergence with $\eta_k = \Theta(\frac{1}{k})$\footnote{For intuition, note that $\sum_{n=1}^\infty \frac{1}{n}$ is divergent.} unless $\lambda_k$ is also decreasing. We have the following result:

\begin{thm}
Let $f$ satisfy Assumption 1, and let $\wt{g}(x)$ be as in \Cref{standard_rate}. Suppose we use stochastic leader selection with $\wt{f}(x)$ having $\opn{Var}(\wt{f}(x)) \leq \sigma_f^2$. If $\eta,\lambda$ are fixed so that $\eta \leq (2M(\nu+1))^{-1}$ and $\eta\lambda \leq (2\kappa)^{-1}, \eta\sqrt{\lambda} \leq (\kappa\sqrt{2m})^{-1}$, then $\limsup_{k \rightarrow \infty} \EX f(x_k) - f(x^\ast) \leq \frac{1}{2}\eta \kappa\sigma^2 + \frac{4}{m} \lambda \sqrt{p} \sigma_f$. If $\eta, \lambda$ decrease at the rate $\eta_k = \Theta(\frac{1}{k}), \lambda_k = \Theta(\frac{1}{k})$, then $\EX f(x_k) - f(x^\ast) \leq O(\frac{1}{k})$.
\end{thm}

The communication period and the accuracy of stochastic leader selection are both methods of reducing the cost of updating the leader, and can be substitutes. When the communication period is long, it may be effective to estimate $f(\ux{x}{i})$ to higher accuracy, since this can be done independently.

\subsection{Non-convex Optimization: Stationary Points}\label{sub:nonconvex}

As mentioned above, EASGD has the flaw that the EASGD objective function can have stationary points such that none of $\ux{x}{1},\ldots,\ux{x}{p},\wt{x}$ is a stationary point of the underlying function $f$. LSGD does not have this issue.

\begin{thm}
Let $\Omega_i$ be the points $(\ux{x}{1},\ldots,\ux{x}{p})$ where $\ux{x}{i}$ is the unique minimizer among $(\ux{x}{1},\ldots,\ux{x}{p})$. 
If $x^\ast = (\ux{w}{1},\ldots,\ux{w}{p}) \in \Omega_i$ is a stationary point of the LSGD objective function, then $\nabla \ux{f}{i}(\ux{w}{i}) = 0$.
\end{thm}

Moreover, it can be shown that for the deterministic algorithm LGD with \emph{any choice of communication periods}, there will always be some variable $\ux{x}{i}$ such that $\liminf \|\nabla f(\ux{x}{i}_k)\| = 0$.

\begin{thm}
Assume that $f$ is bounded below and $M$-Lipschitz-differentiable, and that the LGD step sizes are selected so that $\eta_i < \frac{2}{M}$. Then for any choice of communication periods, it holds that for every $i$ such that $\ux{x}{i}$ is the leader infinitely often, $\liminf_k \|\nabla f(\ux{x}{i}_k)\| = 0$.
\end{thm}

\subsection{Search Direction Improvement from Leader Selection}\label{sub:improve}
In this section, we discuss how LGD can obtain better search directions than gradient descent. In general, it is difficult to determine when the LGD step will satisfy $f(x - \eta(\nabla f(x) + \lambda (x-z))) \leq f(x - \eta \nabla f(x))$, since this depends on the precise combination of $f, x,z, \eta, \lambda$, and moreover, the maximum allowable value of $\eta$ is different for LGD and gradient descent. Instead, we measure the goodness of a search direction by the angle it forms with the Newton direction $d_N(x) = -(\nabla^2 f(x))^{-1}\nabla f(x)$. The Newton method is locally quadratically convergent around local minimizers with non-singular Hessian, and converges in a single step for quadratic functions if $\eta = 1$. Hence, we consider it desirable to have search directions that are close to $d_N$. Let $\theta(u,v)$ denote the angle between $u,v$. Let $d_z = -(\nabla f(x) + \lambda (x-z))$ be the LGD direction with leader $z$, and $d_G(x) = -\nabla f(x)$. The \emph{angle improvement set} is the set of leaders $I_\theta(x, \lambda) = \{z: f(z) \leq f(x), \theta(d_z, d_N(x)) \leq \theta(d_G(x), d_N(x))\}$. The set of candidate leaders is $E = \{z: f(z) \leq f(x)\}$. We aim to show that a large subset of leaders in $E$ belong to $I_\theta(x,\lambda)$.

In this section, we consider the positive definite quadratic $f(x) = \frac{1}{2}x^TAx$ with condition number $\kappa$ and $d_G(x) = -Ax, d_N(x) = -x$. The first result shows that as $\lambda$ becomes sufficiently small, at least half of $E$ improves the angle. We use the $n$-dimensional volume $\opn{Vol}(\cdot)$ to measure the relative size of sets: an ellipsoid $E$ given by $E = \{x: x^TAx \leq 1\}$ has volume $\opn{Vol}(E) = \opn{det}(A)^{-1/2}\opn{Vol}(S_n)$, where $S_n$ is the unit ball.

\begin{thm}\label{shrinking_lambda}
Let $x$ be any point such that $\theta_x = \theta(d_G(x),d_N(x)) > 0$, and let $E = \{ z: f(z) \leq f(x)\}$. Then $\lim_{\lambda \rightarrow 0} \opn{Vol}(I_\theta(x,\lambda)) \geq \frac{1}{2} \opn{Vol}(E)$\footnote{Note that $I_\theta(x,\lambda_1) \supseteq I_\theta(x,\lambda_2)$ for $\lambda_1 \leq \lambda_2$, so the limit is well-defined.}.
%\textcolor{red}{Vol refers to Volume, which for the ellipsoid $\{x: x^TAx \leq 1\}$ is given by $\operatorname{det}(A)^{-1/2}\operatorname{Volume}(S_n)$, where $S_n$ is the unit ball}.
\end{thm}

Next, we consider when $\lambda$ is large. We show that points with large angle between $d_G(x), d_N(x)$ exist, which are most suitable for improvement by LGD. For $r \geq 2$, define $S_r = \{ x: \cos(\theta(d_G(x), d_N(x))) = \frac{r}{\sqrt{\kappa}} \}$. It can be shown that $S_r$ is nonempty for all $r \geq 2$. We show that for $x \in S_r$ for a certain range of $r$, $I_\theta(x,\lambda)$ is at least half of $E$ \emph{for any choice of $\lambda$}.

\begin{thm}\label{large_angle}
Let $R_\kappa = \{r: \frac{r}{\sqrt{\kappa}} + \frac{r^{3/2}}{\kappa^{1/4}} \leq 1\}$. If $x \in S_r$ for $r \in R_\kappa$, then for any $\lambda \geq 0$, $\opn{Vol}(I_\theta(x,\lambda)) \geq \frac{1}{2}\opn{Vol}(E)$.
\end{thm}

Note that \Cref{shrinking_lambda,large_angle} apply only to \emph{convex} functions, or in the neighborhoods of local minimizers where the objective function is locally convex. In nonconvex landscapes, the Newton direction may point towards saddle points \cite{DPGCGB2014NIPS}, which is undesirable; however, since \Cref{shrinking_lambda,large_angle} do not apply in this situation, these results do not imply that LSGD has harmful behavior. For nonconvex problems, our intuition is that many candidate leaders lie in directions of \emph{negative curvature}, which would actually lead away from saddle points, but this is significantly harder to analyze since the set of candidates is unbounded a priori.
\vspace{-0.1in}

\section{Experimental Results}
\label{sec:Exp}

\subsection{Experimental setup}
\begin{wrapfigure}{r}{0.65\textwidth}
\vspace{-0.8in}
% \vspace{-0.15in}
\begin{minipage}[t]{0.65\textwidth}
\hspace{-0.1in}\includegraphics[width=0.52\linewidth]{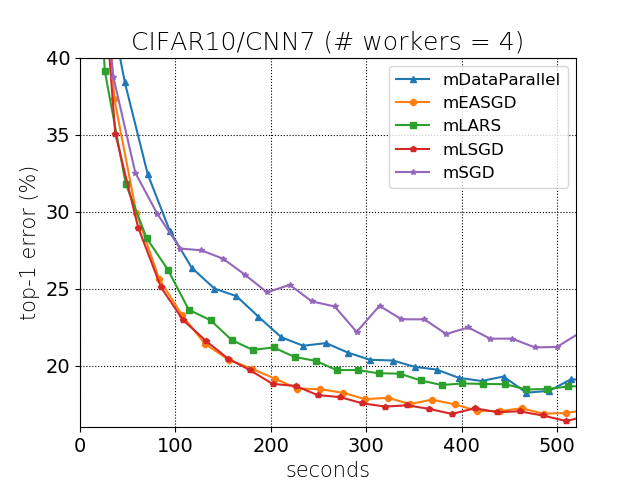}
\includegraphics[width=0.52\linewidth]{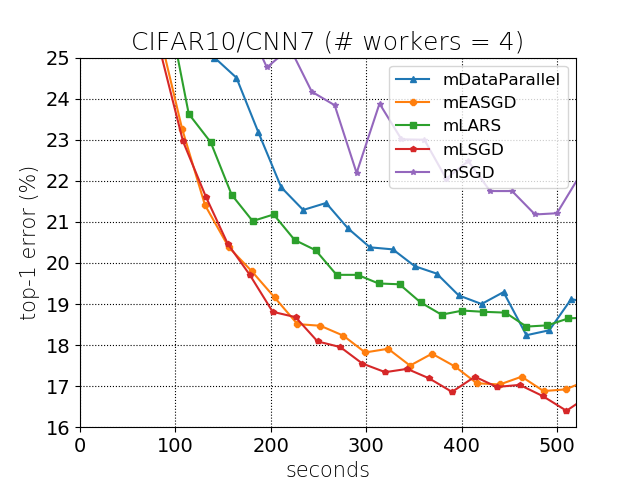}\\
\vspace{-0.1in}
\end{minipage}
\begin{minipage}[t]{0.65\textwidth}
\hspace{-0.1in}\includegraphics[width=0.52\linewidth]{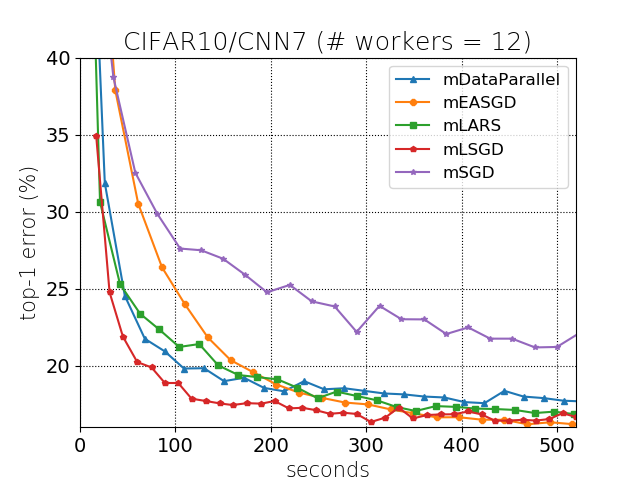}
\includegraphics[width=0.52\linewidth]{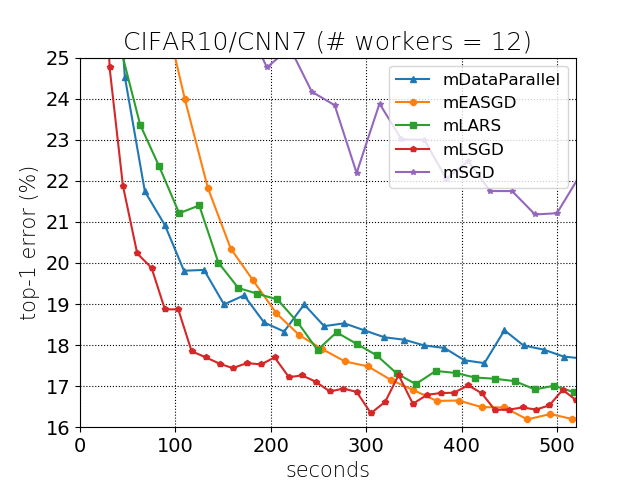}
\vspace{-0.165in}
\caption{CNN$7$ on CIFAR-$10$. Test error for the center variable versus wall-clock time (original plot on the left and zoomed on the right). Test loss is reported in Figure~\ref{fig:CNN7testloss} in the Supplement.}
\label{fig:CNN7}
\end{minipage}
\vspace{-0.15in}
\end{wrapfigure}
In this section we compare the performance of LSGD with state-of-the-art methods for parallel training of deep networks, such as DataParallel\footnote{We refer to PyTorch's official implementation of distributed data parallelism.} (vanilla distributed data parallelization method), LARS, and EASGD (its pseudo-codes can be found in~\cite{EASGD}), as well as sequential technique SGD. All methods use momentum. We use communication period equal to $1$ for DataParallel and LARS in all our experiments as this is the typical setting used for these methods ensuring stable convergence. The experiments were performed using the CIFAR-$10$ data set~\cite{CIFAR} on three benchmark architectures: $7$-layer CNN used in the original EASGD paper (see Section 5.1. in~\cite{EASGD}) that we refer to as CNN$7$, VGG$16$~\cite{Simonyan15}, and ResNet$20$~\cite{He2016DeepRL}; and ImageNet (ILSVRC $2012$) data set~\cite{imagenet_cvpr09} on ResNet$50$.

\begin{wrapfigure}{r}{0.65\textwidth}
\begin{minipage}[t]{0.65\textwidth}
\vspace{-0.1in}
\hspace{-0.1in}\includegraphics[width=0.52\linewidth]{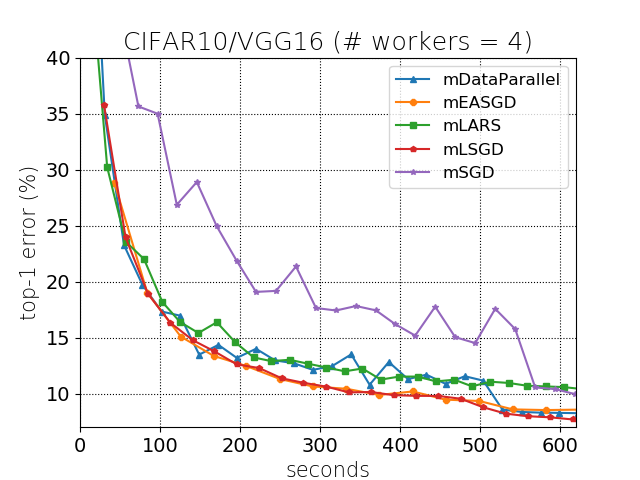}
\includegraphics[width=0.52\linewidth]{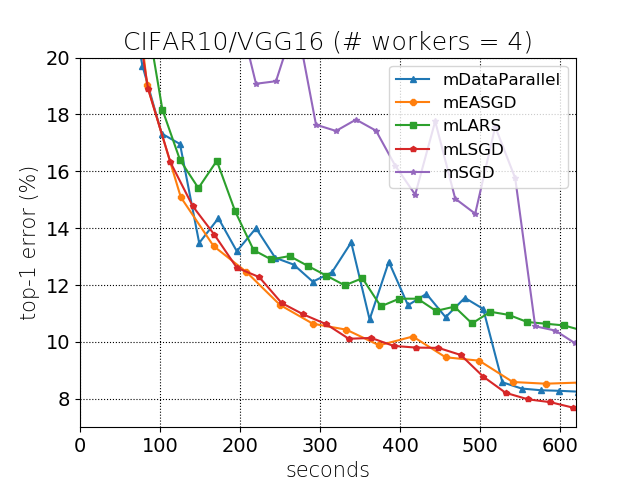}
\vspace{-0.1in}
\end{minipage}
\begin{minipage}[t]{0.65\textwidth}
\vspace{-0.05in}
\hspace{-0.1in}\includegraphics[width=0.52\linewidth]{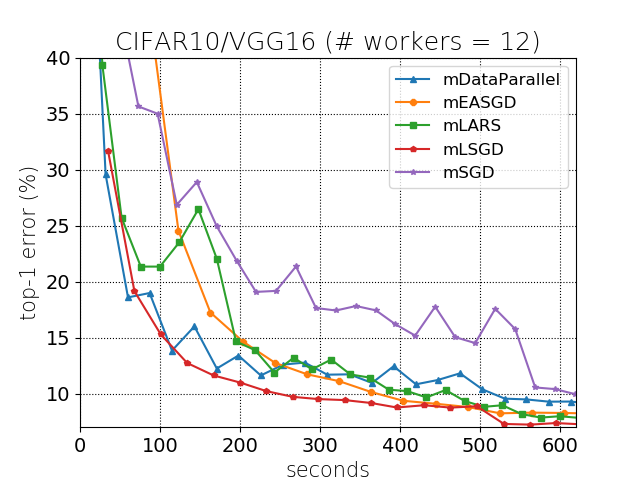}
\includegraphics[width=0.52\linewidth]{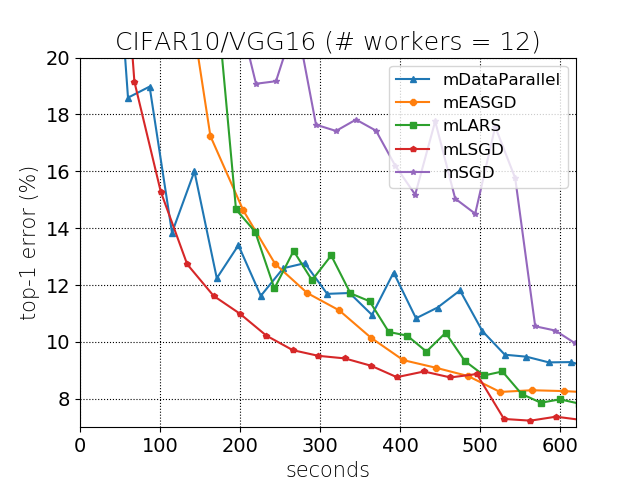}
\vspace{-0.18in}
\caption{VGG$16$ on CIFAR-$10$. Test error for the center variable versus wall-clock time (original plot on the left and zoomed on the right). Test loss is reported in Figure~\ref{fig:VGG16testloss} in the Supplement.}
\label{fig:VGG16}
\vspace{-0.16in}
\end{minipage}
\end{wrapfigure}

During training, we select the leader for the LSGD method based on the average of the training loss computed over the last $10$ (CIFAR-$10$) and $64$ (ImageNet) data batches. At testing, we report the performance of the center variable for EASGD and LSGD, where for LSGD the center variable is computed as the average of the parameters of all workers. [\textit{Remark}: Note that we use the leader's parameter to pull to at training and we report the averaged parameters at testing deliberately. It is demonstrated in our paper (e.g.: Figure \ref{fig:matrix_completion}) that pulling workers to the averaged parameters at training may slow down convergence and we address this problem. Note that after training, the parameters that workers obtained after convergence will likely lie in the same valley of the landscape (see~\cite{baldassi2016unreasonable}) and thus their average is expected to have better generalization ability (e.g.~\cite{DBLP:journals/corr/ChaudhariCSL16, izmailov2018averaging}), which is why we report the results for averaged parameters at testing.] Finally, for all methods we use Nesterov momentum of $0.9$ and weight decay with decay coefficient set to $10^{-4}$. In our experiments we use either $4$ workers (single-leader LSGD setting) or $12$ workers (multi-leader LSGD setting with $3$ groups of workers). For all methods, we report the optimal choice of hyperparameters leading to the smallest achievable test error under similar convergence rates.

\begin{wrapfigure}{r}{0.65\textwidth}
\vspace{-0.15in}
\begin{minipage}[t]{0.65\textwidth}
\hspace{-0.1in}\includegraphics[width=0.52\linewidth]{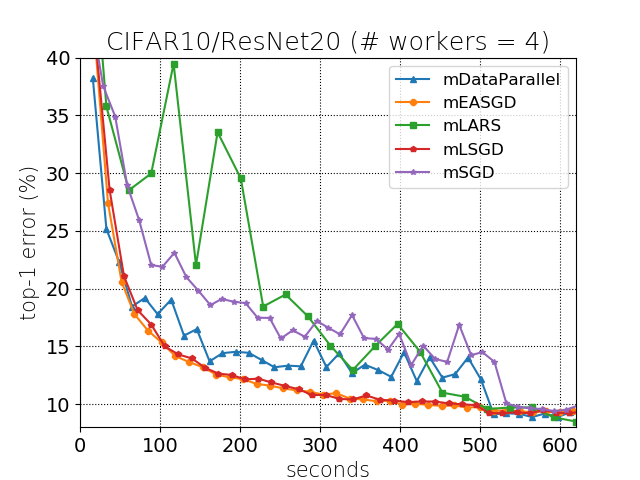}
\includegraphics[width=0.52\linewidth]{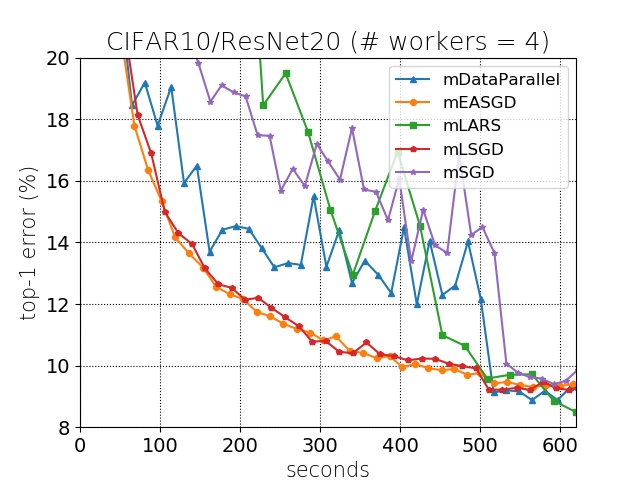}
\vspace{-0.1in}
\end{minipage}
\begin{minipage}[t]{0.65\textwidth}
\hspace{-0.1in}\includegraphics[width=0.52\linewidth]{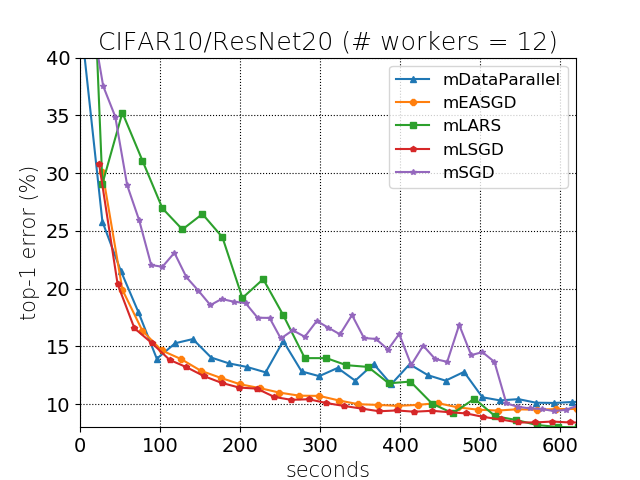}
\includegraphics[width=0.52\linewidth]{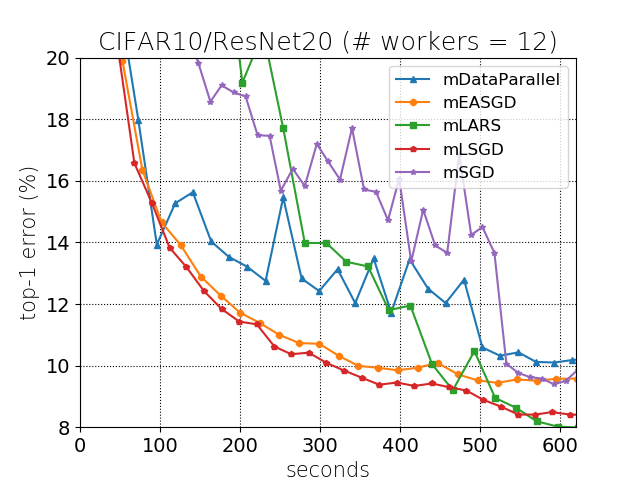}
\vspace{-0.15in}
\caption{ResNet$20$ on CIFAR-$10$. Test error for the center variable versus wall-clock time (original plot on the left and zoomed on the right). Test loss is reported in Figure~\ref{fig:ResNet20testloss} in the Supplement.}
\label{fig:ResNet20}
\end{minipage}
\end{wrapfigure}

We use GPU nodes interconnected with Ethernet. Each GPU node has four GTX 1080 GPU processors where each local worker corresponds to one GPU processor. We use CUDA Toolkit 10.0\footnote{https://developer.nvidia.com/cuda-zone} and NCCL 2\footnote{https://developer.nvidia.com/nccl}. We have developed a software package based on PyTorch for distributed training, which will be released (details are elaborated in Section~\ref{sec:package}).

Data processing and prefetching are discussed in the Supplement. The summary of the hyperparameters explored for each method are also provided in the Supplement. We use constant learning rate for CNN$7$ and learning rate drop (we divide the learning rate by $10$ when we observe saturation of the optimizer) for VGG$16$, ResNet$20$, and ResNet$50$.
\vspace{-0.25in}

\begin{wrapfigure}{r}{0.65\textwidth}
\begin{minipage}[t]{0.65\textwidth}
\vspace{-0.15in}
\hspace{-0.1in}\includegraphics[width=0.52\linewidth]{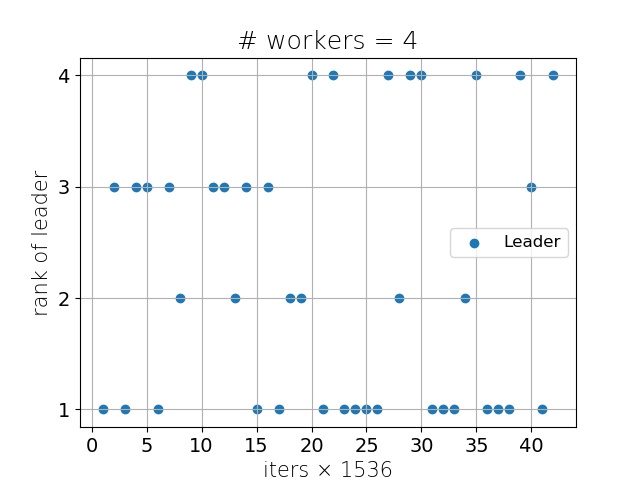}\includegraphics[width=0.52\linewidth]{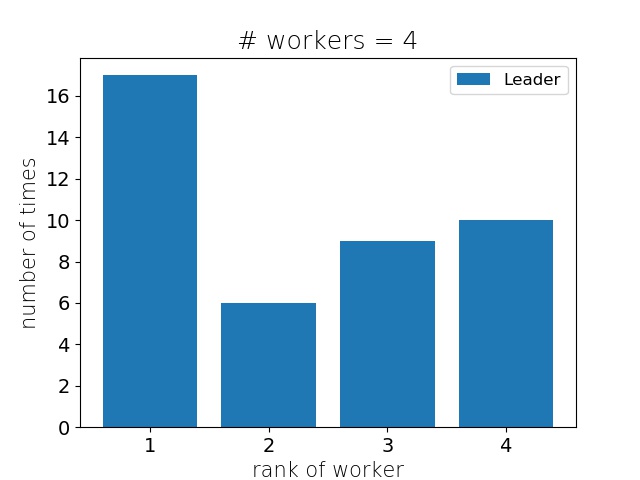} 
\vspace{-0.15in}
\caption{ResNet$20$ on CIFAR-$10$. The identity of the worker that is recognized as the leader (i.e. rank) versus iterations (on the left) and the number of times each worker was the leader (on the right).}
\label{fig:ResNet50rank}
\end{minipage}
\end{wrapfigure}

\vspace{0.25in}
\subsection{Experimental Results}
\label{sec:ExpA}
% In Figure~\ref{fig:CNN7} we report results obtained with CNN$7$ on CIFAR-$10$. We run EASGD and LSGD with communication period $\tau = 64$. We used $\tau_G = 128$  for the multi-leader LSGD case. The number of workers was set to $l = \{4,16\}$. Our method consistently outperforms the competitors in terms of convergence speed (it is roughly $1.5$ times faster than EASGD for $16$ workers) and for $16$ workers it obtains smaller error.

In Figure~\ref{fig:CNN7} we report results obtained with CNN$7$ on CIFAR-$10$. Our method consistently outperforms the competitors in terms of convergence speed (it is roughly $1.5$ times faster than EASGD for $12$ workers).

% In Figure~\ref{fig:VGG16} we demonstrate results for VGG$16$ and CIFAR-$10$ with communication period $64$ and number of workers equal to $4$. LSGD converges marginally faster than EASGD and recovers the same error. At the same time it outperforms significantly DOWNPOUR in terms of convergence speed and obtains a slightly better solution.

In Figure~\ref{fig:VGG16} we demonstrate results for VGG$16$ and CIFAR-$10$. LSGD obtains a much better solution than other methods. Furthermore, LSGD converges marginally faster than EASGD and it outperforms DataParallel and LARS significantly.

% \begin{wrapfigure}{r}{0.65\textwidth}
% \begin{minipage}[t]{0.65\textwidth}
% \vspace{-0.15in}
% \hspace{-0.1in}\includegraphics[width=0.52\linewidth]{CIFAR10/ResNet20/workers=4/leader/rank.png} \includegraphics[width=0.52\linewidth]{CIFAR10/ResNet20/workers=4/leader/number.png}
% \vspace{-0.15in}
% \caption{ResNet$20$ on CIFAR-$10$. The identity of the worker that is recognized as the leader (i.e. rank) versus iterations (on the left) and the number of times each worker was the leader (on the right).}
% \label{fig:ResNet50rank}
% \vspace{-0.16in}
% \end{minipage}
% \end{wrapfigure}

% The experimental results obtained using ResNet$20$ and CIFAR-$10$ for the same setting of communication period and number of workers as in case of CNN$7$ are shown in Figure~\ref{fig:ResNet20}. On $4$ workers we converge comparably fast to EASGD but recover better test error. For this experiment in Figure~\ref{fig:ResNet50rank} we show
% the switching pattern between the leaders indicating that LSGD indeed takes advantage of all workers when exploring the landscape. On $16$ workers we converge roughly $2$ times faster than EASGD and obtain significantly smaller error. In this and CNN$7$ experiment LSGD (as well as EASGD) are consistently better than DONWPOUR and SGD, as expected. 

The experimental results obtained using ResNet$20$ and CIFAR-$10$ are shown in Figure~\ref{fig:ResNet20}. On $4$ workers we converge comparably fast to EASGD but recover better test error. For this experiment in Figure~\ref{fig:ResNet50rank} we show
the switching pattern between the leaders indicating that LSGD indeed takes advantage of all workers when exploring the landscape. On $12$ workers we converge roughly $1.5$ times faster than EASGD and obtain significantly smaller error. In CNN$7$, VGG$16$ and this experiment, LSGD is consistently better than DataParallel and LARS in the aspect of convergence speed, as expected.

\begin{remark}
We believe that these two facts together — (1) the schedule of leader switching recorded in the experiments shows frequent switching, and (2) the leader point itself is not pulled away from minima — suggest that the `pulling away' in LSGD is beneficial: non-leader workers that were pulled away from local minima later became the leader, and thus likely obtained an even better solution than they originally would have.
\end{remark}

% Finally, in Figure~\ref{fig:ResNet50} we report the empirical results for ResNet$50$ run on ImageNet. The number of workers was set to $4$ and the communication period $\tau$ was set to $64$. In this experiment our algorithm behaves comparably to EASGD but converges much faster than DOWNPOUR. Also note that for ResNet50 on ImageNet, SGD is consistently worse than all reported methods (training on ImageNet with SGD on a single GTX1080 GPU until convergence usually takes about a week and gives slightly worse final performance), which is why the SGD curve was deliberately omitted (other methods converge in around two days).

\vspace{-0.1in}
Finally, in Figure~\ref{fig:ResNet50} we report the empirical results for ResNet$50$ run on ImageNet. In this experiment our algorithm behaves comparably to LARS but converges much faster than EASGD for $12$ workers. Also note that for ResNet50 on ImageNet, SGD is consistently worse than all reported methods (training on ImageNet with SGD on a single GTX1080 GPU until convergence usually takes about a week and gives slightly worse final performance), which is why the SGD curve was deliberately omitted (other methods converge in around two days).

\vspace{-0.55in}
\begin{wrapfigure}{r}{0.65\textwidth}
\begin{minipage}[t]{0.65\textwidth}
\vspace{-0.4in}
\hspace{-0.1in}\includegraphics[width=0.52\linewidth]{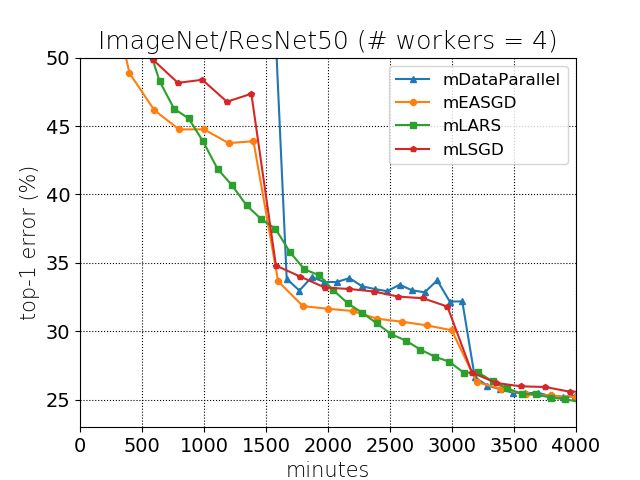}
\includegraphics[width=0.52\linewidth]{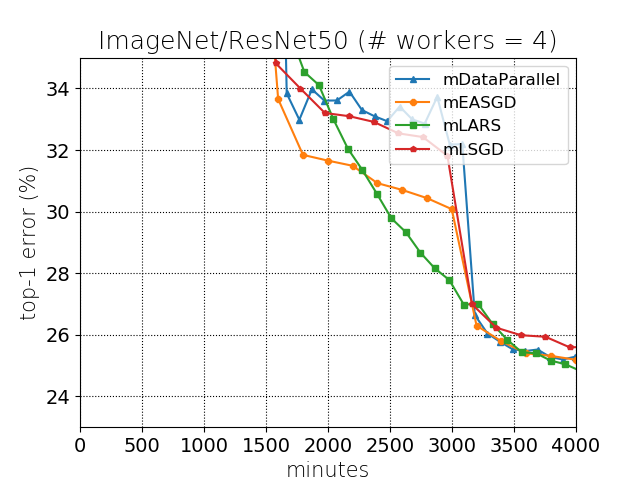}
\vspace{-0.1in}
\end{minipage}
\begin{minipage}[t]{0.65\textwidth}
\hspace{-0.1in}\includegraphics[width=0.52\linewidth]{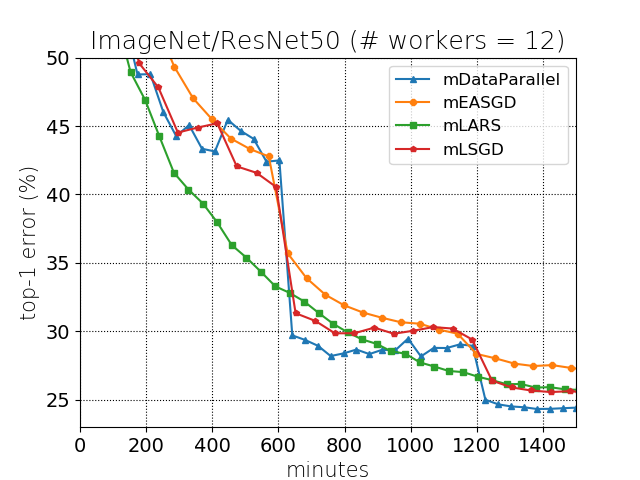}
\includegraphics[width=0.52\linewidth]{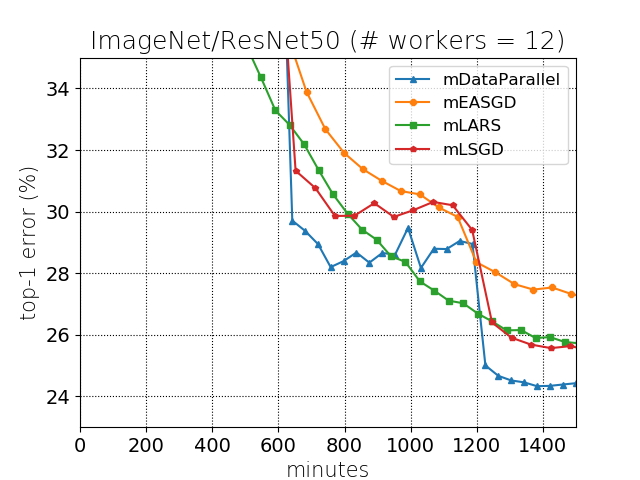}
\caption{ResNet$50$ on ImageNet. Test error for the center variable versus wall-clock time (original plot on the left and zoomed on the right). Test loss is reported in Figure \ref{fig:ResNet50testloss} in the Supplement.}
\label{fig:ResNet50}
\end{minipage}
\vspace{-0.45in}
\end{wrapfigure}

% \textcolor{red}{
% \begin{remark}
% We realize the downgraded performance of LSGD on ImageNet should be majorly blamed on the accuracy of loss estimation. For that part we shall leave it as a future work.
% \end{remark}}

\vspace{0.4in}
\section{Conclusion}
\label{sec:Con}
\vspace{-0.15in}
In this paper we propose a new algorithm called LSGD for distributed optimization in non-convex settings. Our approach relies on pulling workers to the current best performer among them, rather than their average, at each iteration. We justify replacing the average by the leader both theoretically and through empirical demonstrations. We provide a thorough theoretical analysis, including proof of convergence, of our algorithm. Finally, we apply our approach to the matrix completion problem and training deep learning models and demonstrate that it is well-suited to these learning settings.

\vspace{-0.15in}
\section*{Acknowledgements}
WG and DG were supported in part by NSF Grant IIS-1838061. AW acknowledges support from the David MacKay Newton research fellowship at Darwin College, The Alan Turing Institute under EPSRC grant EP/N510129/1 \&  TU/B/000074, and the Leverhulme Trust via the CFI.

\newpage
%\clearpage
\begin{small}
\bibliographystyle{unsrt}
\bibliography{LSGD}
\end{small}

\clearpage
\newpage

\noindent\rule{\textwidth}{4pt}
% \hrule height 4\p
  \vskip 0.25in
  \vskip -\parskip%
\begin{center}
{\LARGE\bf Leader Stochastic Gradient Descent for Distributed Training of Deep Learning Models\\ (Supplementary Material) \par} 
\end{center}
\vskip 0.29in
  \vskip -\parskip
%   \hrule height 1\p
  \vskip 0.09in%
\noindent\rule{\textwidth}{1pt}

\begin{center}
\begin{large}
{\bf Abstract}  \\
\end{large}
\vskip 0.5ex
\end{center}

\begin{quote}
This Supplement presents additional details in support of the full article. These include the proofs of the theoretical statements from the main body of the paper and additional theoretical results. We also provide a toy illustrative example of the difference between LSGD and EASGD. Finally, the Supplement contains detailed description of the experimental setup and additional experiments and figures to provide further empirical support for the proposed methodology.
\end{quote}

\vskip 1ex

\section{LGD versus EAGD: Illustrative Example}
\label{sec:Il}

\begin{figure}[H]
\begin{subfigure}{.5\textwidth}
    \centering
    \includegraphics[width=.9\linewidth]{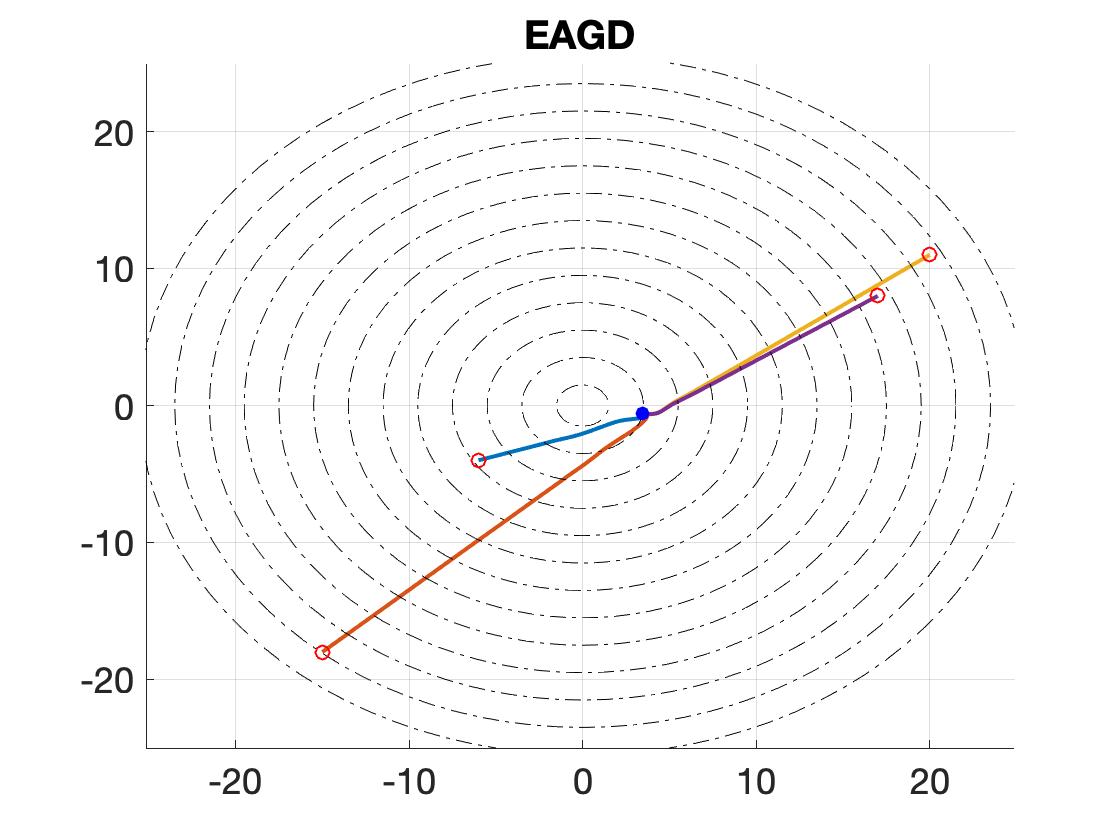}
    %\caption{Trajectory of EASGD}
\end{subfigure}%
\begin{subfigure}{.5\textwidth}
    \centering
    \includegraphics[width=.9\linewidth]{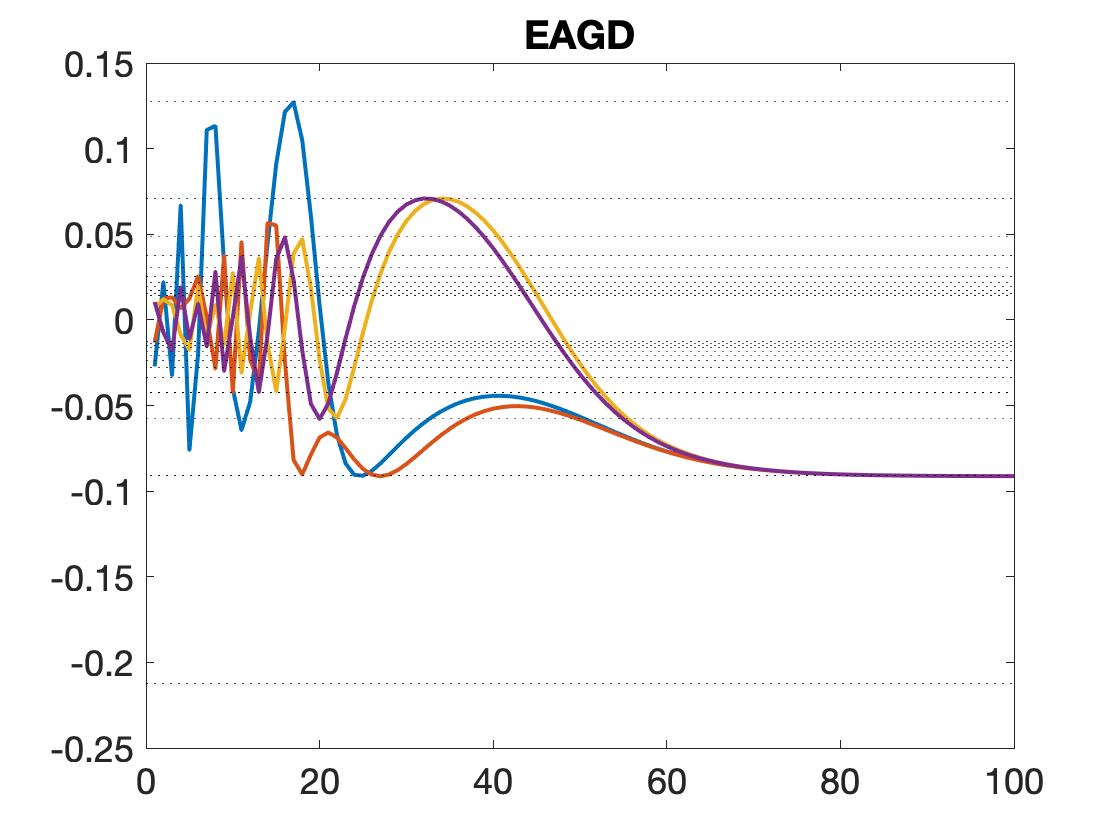}
\end{subfigure}

\begin{subfigure}{.5\textwidth}
    \centering
    \includegraphics[width=.9\linewidth]{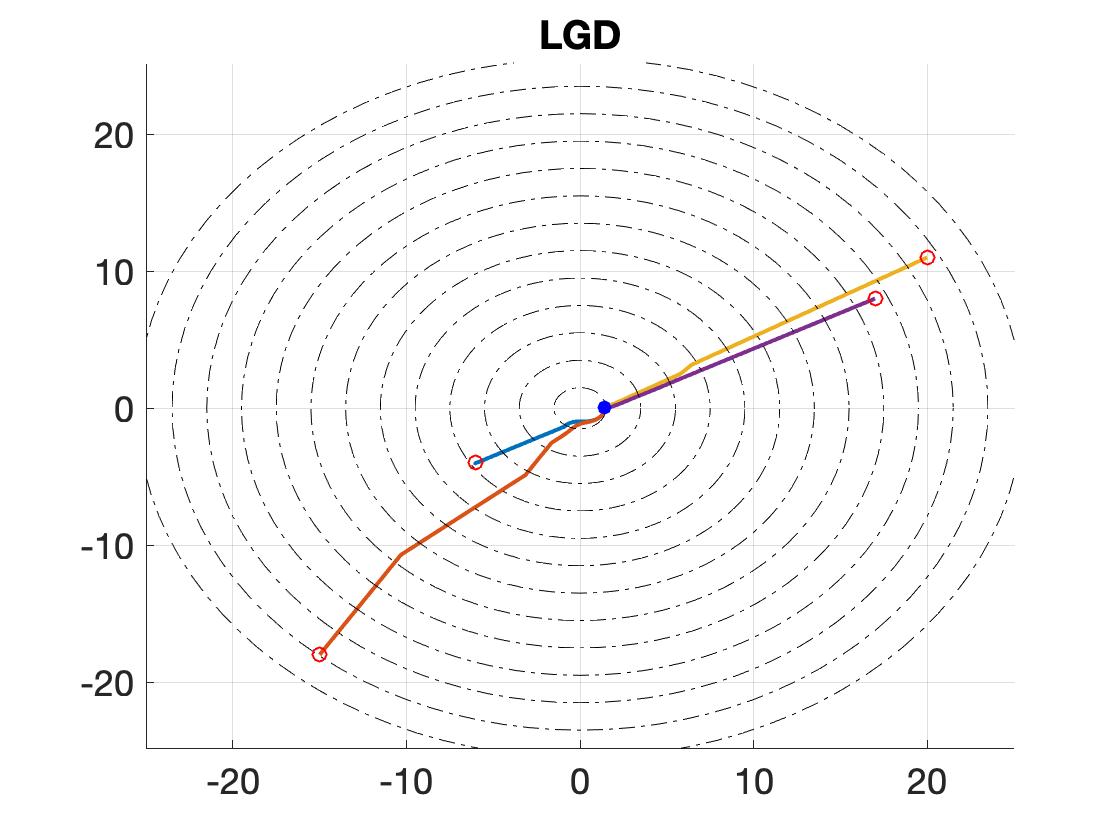}
    %\caption{Trajectory of LSGD}
\end{subfigure}%
\begin{subfigure}{.5\textwidth}
    \centering
    \includegraphics[width=.9\linewidth]{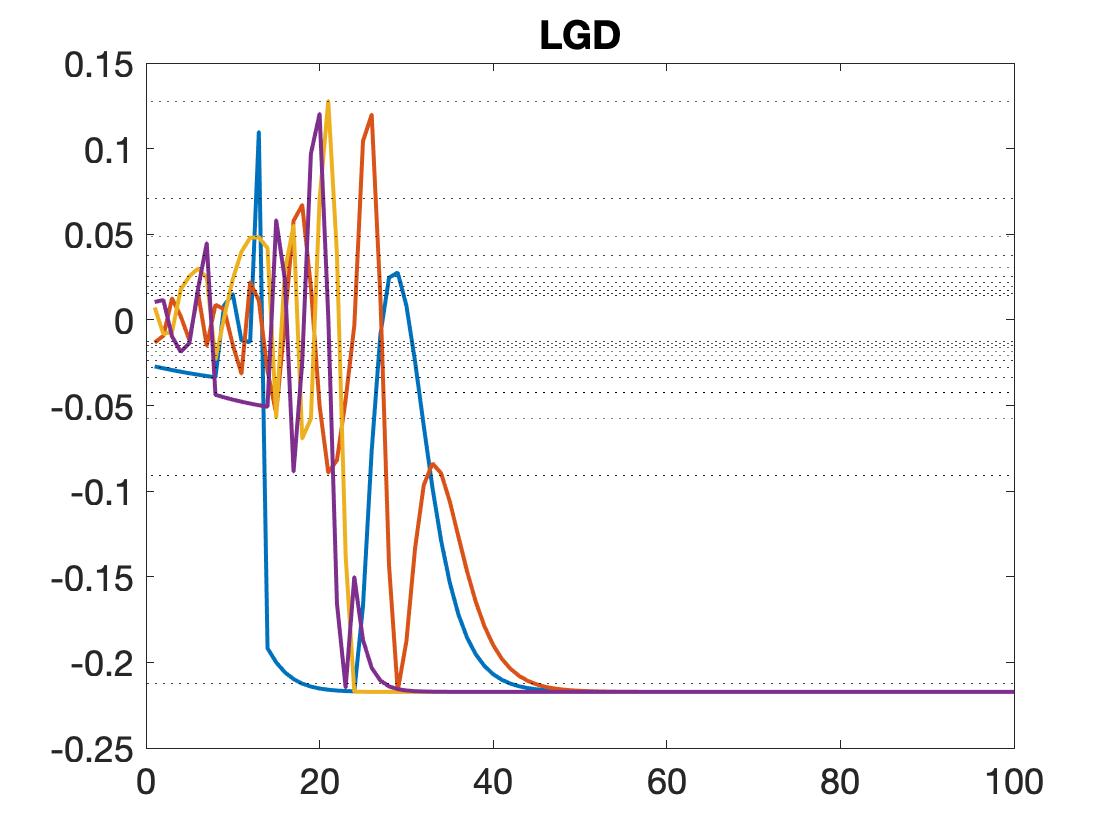}
\end{subfigure}
\caption{\textbf{Left}: Trajectories of variables $(\mathbf{x}, \mathbf{y})$ during optimization. The dashed lines represent the local minima. The red and blue circles are the start and end points of each trajectory, respectively. \textbf{Right}: The value of the objective function $L(x,y)$ for each worker during training.} %\textcolor{red}{Change EASGD to EAGD and SOSGD to LGD.}
\label{res:sinc}
\end{figure}

We consider the following non-convex optimization problem:\\
\[\min_{x,y} L(x,y), \:\:\:\text{where}\:\:\: L(x,y) = \frac{\sin(\sqrt{x^2 + y^2} \cdot \pi)}{\sqrt{x^2 + y^2} \cdot \pi}.
\]
%$x$ and $y$ are initialized as follows: $x \sim Uniform(-25,25)$ and $y \sim Uniform(-25,25)$.

Both methods use $4$ workers with initial points $(-6,-4)$, $(-15,-18)$, $(20,11)$ and $(17, 8)$. The communication period is set to $1$. The learning rate for both EAGD and LGD equals $0.1$. Furthermore, EAGD uses $\beta = 0.43$ and LGD uses $\lambda = 0.1$.

Table~\ref{res:sinc_t} captures optima obtained by different methods.

\begin{table}[htp!]
\centering
\begin{tabular}{||c|c||}
\hline
Optimizer  &$L(x,y)$ \\
\hline
EAGD  &-0.0912 \\ 
LGD  &\textbf{-0.2172} \\
\hline
\end{tabular}
\vspace{0.1in}
\caption{Optimum $L(x^*,y^*)$ recovered by EAGD and LGD.}
\label{res:sinc_t}
\end{table}

Figure~\ref{res:sinc} captures the optimization trajectories of EAGD and LGD algorithms. Clearly, EAGD suffers from the averaging policy, whereas LGD is able to recover a solution close to the \textit{global optimum}. 

\section{Proofs of Theoretical Results}
We provide omitted proofs from the main text.

\subsection{Definitions and Notation}
Recall that the objective function of Leader (Stochastic) Gradient Descent (L(S)GD) is defined as
\begin{equation}\label{eq:lsgd_obj}
\min_{\ux{x}{1},\ldots,\ux{x}{p}} L(\ux{x}{1},\ldots,\ux{x}{p}) := \sum_{i=1}^p f(\ux{x}{i}) + \frac{\lambda}{2}\|\ux{x}{i}- \wt{x}\|^2
\end{equation}
where $\wt{x} = \argmin \{ f(\ux{x}{1}),\ldots,f(\ux{x}{p})\}$. An L(S)GD step is a (stochastic) gradient step applied to $L$. Writing $z = \wt{x}$ at a particular $(\ux{x}{1},\ldots,\ux{x}{n})$, the update in the variable $\ux{x}{i}$ is
$$\ux{x}{i}_+ = \ux{x}{i} - \eta(\nabla f(\ux{x}{i}) + \lambda (\ux{x}{i} - z))$$
Observe that this reduces to a (S)GD step for the variable which is the leader.

Practical variants of the algorithm do not communicate the updated leader at every iteration. Thus, in our analysis, we will generally take $z$ to be an arbitrary guiding point, which is not necessarily the minimizer of $\ux{x}{1},\ldots,\ux{x}{p}$, nor even satisfy $f(z) \leq f(\ux{x}{i})$ for all $i$. The required properties of $z$ will be specified on a result-by-result basis.

When discussing the optimization landscape of LSGD, the term `LSGD objective function' will refer to \eqref{eq:lsgd_obj} with $\wt{x}$ defined as the argmin.

\emph{Communication periods} are sequences of steps where the leader is not updated. We introduce the notation $x_{k,j}$ for the $j$-th step in the $k$-th period, where the leader $z$ is updated only at the beginning of each period. We write $b_i(k)$ for the number of steps that $\ux{x}{i}$ takes during the $k$-th period. The standard LSGD defined above has $b_i(k) = 1$ for all $i,k$, in which case $\ux{x}{i}_{k,1} = \ux{x}{i}_k$. In addition, let $\wt{x}_k = \opn{argmin} \{f(\ux{x}{1}_{k,1}), \ldots, f(\ux{x}{p}_{k,1})\}$, the leader for the $k$-th period.

\subsection{Stationary Points of EASGD}
The EASGD \cite{EASGD} objective function is defined as
\begin{equation}
\min_{\ux{x}{1},\ldots,\ux{x}{p}, \wt{x}} L(\ux{x}{1},\ldots,\ux{x}{p}, \wt{x}) := \sum_{i=1}^p f(\ux{x}{i}) + \frac{\lambda}{2}\|\ux{x}{i}- \wt{x}\|^2.
\end{equation}
Observe that unlike LSGD, $\wt{x}$ is a decision variable of EASGD. A stationary point of EASGD is a point such that $\nabla L(x^1,\ldots,x^p, \wt{x}) = 0$.

\begin{prop}
There exists a Lipschitz differentiable function $f: \RR \rightarrow \RR$ such that for every $0 < \lambda \leq 1$, there exists a point $(x_\lambda,y_\lambda,0)$ which is a stationary point of EASGD with parameter $\lambda$, but none of $\{x_\lambda,y_\lambda, 0\}$ is a stationary point of $f$.
\label{prop:EASGD}
\end{prop}
\begin{proof}
Define $f(x)$ by
$$f(x) = \left\{ \begin{array}{cl} e^{x+1} & \text{ if } x < -1\\
p(x) & \text{ if } -1 \leq x \leq 1 \\
e^{-x+1} & \text{ if } x > 1 \end{array} \right.$$
where $p(x) = a_6x^6 + \ldots + a_1x + a_0$ is a sixth-degree polynomial. For $f$ to be Lipschitz differentiable, we will select $p(x)$ to make $f$ twice continuously differentiable, with bounded second derivative. To make $f$ twice continuously differentiable, we must have $p(1) = 1, p'(1) = -1, p''(1) = 1$ and $p(-1) = -1, p'(-1) = 1, p''(-1) = -1$. Since we aim to have $f'(0) \neq 0$, we also will require $f'(0) = p'(0) = 1$. The existence of $p$ is equivalent to the solvability of a linear system, which is easily checked to be invertible. Thus, we deduce that such a function $f$ exists.
%$$\begin{pmatrix} 1 & 1 & 1 & 1 & 1 & 1 & 1 \\
%1 & -1 & 1 & -1 & 1 & -1 & 1 \\
%6 & 5 & 4 & 3 & 2 & 1 & 0 \\
%6 & -5 & 4 & -3 & 2 & -1 & 0 \\
%30 & 20 & 12 & 6 & 2 & 0 & 0 \\
%30 & -20 & 12 & -6 & 2 & 0 & 0 \\
%0 & 0 & 0 & 0 & 0 & 0 & 1 \end{pmatrix}\begin{pmatrix} a_6 \\ a_5 \\ a_4 \\ a_3 \\ a_2 \\ a_1 \\ a_0 \end{pmatrix} = \begin{pmatrix} 1 \\ -1 \\ -1 \\ 1 \\ 1 \\ -1 \\ 1 \end{pmatrix}$$

It remains to show that for any $0 < \lambda \leq 1$, there exists a stationary point $(x,y, 0)$ of EASGD. Set $x = -y$. The first-order condition yields $f'(x) + \lambda x = 0$. Since $\lambda \leq 1$, we have $\lambda(1) + f'(1) \leq 0$. For $x \geq 1$, $f'(x) = -e^{-x+1}$ is an increasing function, so $f'(x) + \lambda x$ is increasing, and we deduce that there exists a solution $y_\lambda \geq 1$ with $\lambda y_\lambda+ f'(y_\lambda) = 0$. By symmetry, $-y_\lambda \leq -1$ satisfies $f'(-y_\lambda) + \lambda(-y_\lambda) = 0$, since $f'(x) = e^{x+1}$ for $x \leq -1$. Hence, $(-y_\lambda, y_\lambda, 0)$ is a stationary point of EASGD, but none of $\{-y_\lambda, y_\lambda, 0\}$ are stationary points of $f$.
\end{proof}

\subsection{Technical Preliminaries}
Recall the statement of Assumption 1:
\begin{description}
\item[Assumption 1] $f$ is $M$-Lipschitz-differentiable and $m$-strongly convex, which is to say, the gradient $\nabla f$ satisfies $\|\nabla f(x) - \nabla f(y)\| \leq M\|x - y\|$, and $f$ satisfies
$$f(y) \geq f(x) + \nabla f(x)^T(y-x) + \frac{m}{2}\|y-x\|^2.$$
We write $x^\ast$ for the unique minimizer of $f$, and $\kappa := \frac{M}{m}$ for the condition number of $f$.
\end{description}

We will frequently use the following standard result.
\begin{lma}\label{standard_lipschitzg}
If $f$ is $M$-Lipschitz-differentiable, then
$$f(y) \leq f(x) + \nabla f(x)^T(y-x) + \frac{M}{2}\|y - x\|^2.$$
\end{lma}
\begin{proof}
See \cite[eq. (4.3)]{BCN2018SIAMREV}.
\end{proof}

\begin{lma}\label{basic_inequality}
Let $f$ be $m$-strongly convex, and let $x^\ast$ be the minimizer of $f$. Then
\begin{equation}\label{eq:pl_inequality}
f(w) - f(x^\ast) \leq \frac{1}{2m}\|\nabla f(w)\|^2
\end{equation}
and
\begin{equation}\label{eq:lb_f}
f(w) - f(x^\ast) \geq \frac{m}{2}\|w - x^\ast\|^2
\end{equation}
\end{lma}
\begin{proof}
\Cref{eq:pl_inequality} is the well-known Polyak-\L{}ojasiewicz inequality. \Cref{eq:lb_f} follows from the definition of strong convexity, and $\nabla f(x^\ast) = 0$.
\end{proof}

\begin{lma}\label{gd_descent}
Let $f$ be $M$-Lipschitz-differentiable. If the gradient descent step size $\eta < \frac{2}{M}$, then $\|\nabla f(x)\|^2 \leq \alpha (f(x) - f(x^+))$, where $\alpha = \frac{2}{\eta(2 - \eta M)}$.
\end{lma}
\begin{proof}
By \Cref{standard_lipschitzg},
\begin{align*}
f(x_+) &\leq f(x) - \eta \|\nabla f(x)\|^2 + \frac{\eta^2}{2}M \|\nabla f(x)\|^2 \\
&= f(x) - \frac{\eta}{2}(2 - \eta M) \|\nabla f(x)\|^2
\end{align*}
Rearranging yields the desired result.
\end{proof}

\subsection{Proofs from \Cref{subsub:rates}}

\begin{lma}[One-Step Descent]\label{main_descent}
Let $f$ satisfy Assumption 1. Let $\wt{g}(x)$ be an unbiased estimator for $\nabla f(x)$ with $\opn{Var}(\wt{g}(x)) \leq \sigma^2 + \nu \|\nabla f(x)\|^2$. Let $x$ be the current iterate, and let $z$ be another point, with $\delta := x - z$. The LSGD step $x_+ = x - \eta(\wt{g}(x) + \lambda(x -z))$ satisfies:
\begin{align}\label{eq:main_descent}
\EX f(x_+) &\leq  f(x) - \frac{\eta}{2} (1 - \eta M(\nu+1))\|\nabla f(x)\|^2 - \frac{\eta}{4}\lambda(m - 2\eta M\lambda) \|\delta\|^2 \\
\nonumber &\bump -\frac{\eta\sqrt{\lambda}}{\sqrt{2}}( \sqrt{m} - \eta M\sqrt{2\lambda})\|\nabla f(x)\|\|\delta\| - \eta\lambda(f(x) - f(z)) + \frac{\eta^2 }{2} M\sigma^2
\end{align}
where the expectation is with respect to $\wt{g}(x)$, and conditioned on the current point $x$. Hence, for sufficiently small $\eta,\lambda$ with $\eta \leq (2M(\nu+1))^{-1}$ and $\eta\lambda \leq (2\kappa)^{-1}, \eta\sqrt{\lambda} \leq (\kappa\sqrt{2m})^{-1}$,
\begin{equation}\label{eq:main_descent_simple}
\EX f(x_+) - f(x^\ast) \leq (1-m\eta)(f(x) - f(x^\ast)) - \eta\lambda( f(x) - f(z)) + \frac{\eta^2 M}{2}\sigma^2
\end{equation}
\end{lma}
\begin{proof}
The proof is similar to the convergence analysis of SGD. We apply \Cref{standard_lipschitzg} to obtain
\begin{align*}
f(x_+) &\leq f(x) - \eta \nabla f(x)^T(\wt{g}(x) + \lambda\delta) + \frac{\eta^2}{2} M \|\wt{g}(x) + \lambda\delta\|^2.
\end{align*}
Taking the expectation and using $\EX \wt{g}(x) = \nabla f(x)$,
\begin{align*}
\EX f(x_+) &\leq f(x) - \eta \|\nabla f(x)\|^2 - \eta\lambda \nabla f(x)^T\delta + \frac{\eta^2\lambda^2}{2} M\|\delta\|^2 + \eta^2\lambda M \nabla f(x)^T\delta + \frac{\eta^2}{2} M \EX \lb \wt{g}(x)^T\wt{g}(x) \rb
\end{align*}
Using the definition of $m$-strong convexity, we have $f(z) \geq f(x) - \nabla f(x)^T\delta + \frac{m}{2}\|\delta\|^2$, 
from which we deduce that $-\nabla f(x)^T\delta \leq -(f(x) - f(z) + \frac{m}{2}\|\delta\|^2)$. Substituting this above, and splitting both the terms $\eta \|\nabla f(x)\|^2, \frac{\eta}{2}m\lambda\|\delta\|^2$ in half, we obtain 
\begin{align*}
\EX f(x_+) &= f(x) - \frac{\eta}{2} \|\nabla f(x)\|^2 + \frac{\eta^2}{2} M \EX \lb \wt{g}(x)^T\wt{g}(x) \rb \\
&\bump - \frac{\eta}{4} m\lambda\|\delta\|^2 + \frac{\eta^2}{2}\lambda^2 M \|\delta\|^2 \\
&\bump - \frac{\eta}{2} \|\nabla f(x)\|^2 - \frac{\eta }{4} m\lambda \|\delta\|^2 + \eta^2\lambda M \nabla f(x)^T\delta \\
&\bump - \eta\lambda(f(x) - f(z))
\end{align*}

We proceed to bound each line. For the first line, the standard bias-variance decomposition yields
$$ \EX \lb \wt{g}(x)^T\wt{g}(x) \rb  \leq (\nu+1) \|\nabla f(x)\|^2 + \sigma^2$$
and so we have
$$- \frac{\eta}{2} \|\nabla f(x)\|^2 + \frac{\eta^2}{2} M \EX \lb \wt{g}(x)^T\wt{g}(x) \rb \leq -\frac{\eta}{2} (1 - \eta M(\nu+1))\|\nabla f(x)\|^2 + \frac{\eta^2}{2} M\sigma^2.$$

For the second line, we obtain
$$-\frac{\eta}{4}m\lambda\|\delta\|^2 + \frac{\eta^2}{2}\lambda^2 M \|\delta\|^2 \leq -\frac{\eta}{4}\lambda( m - 2\eta M\lambda) \|\delta\|^2.$$
For the third line, we apply the inequality $a^2 + b^2 \geq 2ab$ to obtain
$$\frac{\eta}{2} \|\nabla f(x)\|^2 + \frac{\eta}{4}m\lambda \|\delta\|^2 \geq \frac{\eta}{\sqrt{2}} \sqrt{m\lambda} \|\nabla f(x)\|\|\delta\|.$$
Using the Cauchy-Schwarz inequality, we then obtain
$$-\frac{\eta}{2} \|\nabla f(x)\|^2 - \frac{\eta}{4}m\lambda \|\delta\|^2 + \eta^2\lambda \nabla Mf(x)^T\delta \leq -\frac{\eta\sqrt{\lambda}}{\sqrt{2}}(\sqrt{m} - \eta M\sqrt{2\lambda})\|\nabla f(x)\|\|\delta\|.$$

Combining these inequalities yields the desired result.
\end{proof}

\begin{thm}\label{main_descent_standard_rates}
Let $f$ satisfy Assumption 1. Suppose that the leader $z_k$ is always chosen so that $f(z_k) \leq f(x_k)$. If $\eta,\lambda$ are fixed so that $\eta \leq (2M(\nu+1))^{-1}$ and $\eta\lambda \leq (2\kappa)^{-1}, \eta\sqrt{\lambda} \leq (\kappa\sqrt{2m})^{-1}$, then $\limsup\limits_{k \rightarrow \infty} \EX f(x_k) - f(x^\ast) \leq \frac{1}{2}\eta\kappa\sigma^2$. If $\eta$ decreases at the rate $\eta_k = \Theta(\frac{1}{k})$, then $\EX f(x_k) - f(x^\ast) = O(\frac{1}{k})$.
\end{thm}
\begin{proof}
This result follows \eqref{eq:main_descent_simple} and Theorems 4.6 and 4.7 of \cite{BCN2018SIAMREV}.
\end{proof}

\subsection{Proofs from \Cref{subsub:comm}}

\begin{thm}
Let $f$ satisfy Assumption 1. Suppose that $\eta,\lambda$ are small enough that $\eta\lambda \leq 1$ and $\eta \leq (2M(\nu+1))^{-1}, \eta\lambda \leq (2\kappa)^{-1}, \eta\sqrt{\lambda} \leq (\kappa\sqrt{2m})^{-1}$. If $f(x) \leq f(z)$, then $\EX f(x_+) \leq  f(z) + \frac{1}{2}\eta^2 M\sigma^2$.
\end{thm}
\begin{proof}
This follows from \eqref{main_descent_standard_rates}, by combining $f(x) - \eta\lambda(f(x) - f(z))$, and using $f(z) \geq f(x)$.
\end{proof}

\begin{thm}
Let $f$ be $m$-strongly convex, and let $x^\ast$ be the minimizer of $f$. Fix a constant $\lambda$ and any point $z$, and define the function $\psi(x) = f(x) + \frac{\lambda}{2}\|x - z\|^2$. Since $\psi$ is strongly convex, it has a unique minimizer $w$. The minimizer $w$ satisfies
\begin{equation}
f(w) - f(x^\ast) \leq \frac{\lambda}{m+\lambda}(f(z) - f(x^\ast))
\end{equation}
and\footnote{If we also assume that $f$ is Lipschitz-differentiable (that is, $\nabla^2 f(x) \preceq MI$), then we can obtain a similar inequality to the second directly from the first, but this is generally weaker than the bound given here.}
\begin{equation}
\|w - x^\ast\|^2 \leq \frac{\lambda^2}{m(m+\lambda)}\|z - x^\ast\|^2
\end{equation}
\end{thm}
\begin{proof}
The first-order condition for $w$ implies that $\nabla f(w) + \lambda(w-z) = 0$, so $\lambda^2\|w - z\|^2 = \|\nabla f(w)\|^2$. Combining this with the Polyak-\L{}ojasiewicz inequality, we obtain
$$\frac{\lambda}{2}\|w - z\|^2 = \frac{1}{2\lambda}\|\nabla f(w)\|^2 \geq \frac{m}{\lambda} (f(w) - f(x^\ast))$$
We have $\psi(w) \leq \psi(z) = f(z)$, so $f(w) - f(x^\ast) \leq f(z) - f(x^\ast) - \frac{\lambda}{2}\|w - z\|^2$. Substituting, $f(w) - f(x^\ast) \leq f(z) - f(x^\ast) - \frac{m}{\lambda}(f(w) - f(x^\ast))$, which yields the first inequality.

We also have $\psi(w) = f(w) + \frac{\lambda}{2}\|w - z\|^2 \leq \psi(x^\ast) = f(x^\ast) + \frac{\lambda}{2}\|x^\ast - z\|^2$, whence $f(w) - f(x^\ast) \leq \frac{\lambda}{2}(\|x^\ast - z\|^2 - \|w - z\|^2)$. Hence, we have
\begin{align*}
f(w) - f(x^\ast) &\leq \frac{\lambda}{2}(\|x^\ast - z\|^2 - \|w - z\|^2) \\
&\leq \frac{\lambda}{2}\|z -x^\ast\|^2 - \frac{m}{\lambda}(f(w) - f(x^\ast))
\end{align*}
so $f(w) - f(x^\ast) \leq \frac{\lambda^2}{2(m + \lambda)}\|z - x^\ast\|^2$. Finally, by \Cref{basic_inequality}, $f(w) - f(x^\ast) \geq \frac{m}{2}\|w - x^\ast\|^2$, which yields the result.
\end{proof}

\subsection{Proofs from \Cref{subsub:stochleader}}\label{sub:proofs_stochleader}

We first present two lemmas which consider the problem of selecting the minimizer from a collection, based on a single estimate of the value of each item.

\begin{lma}\label{prob_y_bound}
Let $\mu_1 \leq \mu_2 \leq \ldots \leq \mu_p$. Suppose that $Y_1,\ldots,Y_p$ is a collection of random variables with $\EX Y_i = \mu_i$ and $\opn{Var}(Y_i) \leq \sigma^2$. Let $\wt{\mu} = \mu_m$ where $m = \opn{argmin} \{Y_1,\ldots,Y_p\}$. Then
$$\Pr(\wt{\mu} \geq \mu_k) \leq 4\sigma^2\sum_{i=k}^p \frac{1}{(\mu_i - \mu_1)^2}$$
Therefore, for any $a \geq 0$,
$$\Pr(\wt{\mu} \geq \mu_1 + a) \leq 4\sigma^2 \frac{p}{a^2}.$$
\end{lma}
\begin{proof}
In order for $\mu_m \geq \mu_k$, we must have $Y_j \leq Y_1$ for some $j \geq k$. Thus, $\{\wt{\mu} \geq \mu_k\}$ is a subset of the event $\{Y_1 \geq \min \{ Y_k,\ldots,Y_p \} \}$. Taking the union bound,
$$\Pr(Y_1 \geq \min\{Y_k,\ldots,Y_p\}) \leq \sum_{i=k}^p \Pr(Y_1 \geq Y_i)$$
Applying Chebyshev's inequality to $Y_1 - Y_i$, and noting that $\opn{Var}(Y_1 - Y_i) \leq 4\sigma^2$ (if $Y_1,Y_i$ are independent, then this can be tightened to $2\sigma^2$), we have
\begin{align*}
\Pr(Y_1 - Y_i \geq 0) &\leq \Pr(|Y_1 - Y_i - (\mu_i-\mu_1)| \geq \mu_i - \mu_1) \leq \frac{4\sigma^2}{(\mu_i - \mu_1)^2}.
\end{align*}
\end{proof}

\begin{lma}\label{exp_y_bound}
Let $\wt{\mu}$ be defined as in \Cref{prob_y_bound}. Then
$$\EX \wt{\mu} - \mu_1 \leq 4\sqrt{p}\sigma$$
\end{lma}
\begin{proof}
Recall that the expected value of a non-negative random variable $Z$ can be expressed as $\EX Z = \int_0^\infty \Pr(Z \geq t) dt$. We apply this to the variable $\wt{\mu} - \mu_1$. Using \Cref{prob_y_bound}, we obtain, for any $a > 0$,
\begin{align*}
\EX \wt{\mu} - \mu_1 = \int_0^\infty \Pr(\wt{\mu} - \mu_1 \geq t)dt &= \int_0^a \Pr(\wt{\mu} - \mu_1 \geq t)dt + \int_a^\infty \Pr(\mu^\ast - \mu_1 \geq t)dt \\
&\leq a + \int_a^\infty \Pr(\wt{\mu} - \mu_1 \geq t) dt \\
&\leq a + \int_a^\infty 4\sigma^2 \frac{p}{t^2} dt = a + 4\sigma^2 \frac{p}{a}
\end{align*}
The AM-GM inequality implies that $a + 4\sigma^2\frac{p}{a} \geq 4\sqrt{p}\sigma$, with equality when $a = 2\sqrt{p}\sigma$.
\end{proof}

We now apply this to stochastic leader selection in LSGD, where $\mu_i$ corresponds to the true value $f(\ux{x}{i})$, and $Y_i$ is a function estimator.

\begin{lma}
Let $f$ satisfy Assumption 1. Suppose that LSGD has a gradient estimator with $\opn{Var}(\wt{g}(x)) \leq \sigma^2 + \nu\|\nabla f(x)\|^2$ and selects the stochastic leader with a function estimator $\wt{f}(x)$ with $\opn{Var}(\wt{f}(x)) \leq \sigma_f^2$. Then, taking the expectation with respect to the gradient estimator and the stochastic leader $z$, we have
\begin{align*}
\EX f(x_+) &\leq  f(x) + 4\eta\lambda \sqrt{p} \sigma_f + \frac{\eta^2 }{2} M\sigma^2   \\
\nonumber &\bump - \frac{\eta}{2} (1 - \eta M(\nu+1))\|\nabla f(x)\|^2 - \frac{\eta}{4}\lambda(m - 2\eta M\lambda) \|\delta\|^2 -\frac{\eta\sqrt{\lambda}}{\sqrt{2}}( \sqrt{m} - \eta M\sqrt{2\lambda})\|\nabla f(x)\|\|\delta\|
\end{align*}
\end{lma}
\begin{proof}
From \Cref{main_descent}, we obtain
\begin{align*}
\EX f(x_+) &\leq  f(x) - \frac{\eta}{2} (1 - \eta M(\nu+1))\|\nabla f(x)\|^2\\
\nonumber &\bump - \frac{\eta}{4}\lambda(m - 2\eta M\lambda) \|\delta\|^2 \\
\nonumber &\bump -\frac{\eta\sqrt{\lambda}}{\sqrt{2}}( \sqrt{m} - \eta M\sqrt{2\lambda})\|\nabla f(x)\|\|\delta\| \\
\nonumber &\bump - \eta \lambda (f(x) - \EX f(z)) + \frac{\eta^2 }{2} M\sigma^2 
\end{align*}
Note that in the last line, we have $\EX f(z)$ because $z$ is now stochastic. Applying \Cref{exp_y_bound} to the stochastic leader, we obtain $\EX f(z) \leq f(z_{true}) + 4\sqrt{p}\sigma_f$. The true leader satisfies $f(z_{true}) \leq f(x)$ by definition. Hence $f(x) - \EX f(z) \geq f(x) - f(z_{true}) - 4\sqrt{p}\sigma_f \geq -4\sqrt{p}\sigma_f$, and so $-\eta \lambda (f(x) - \EX f(z)) \leq 4\eta\lambda \sqrt{p}\sigma_f$. 
\end{proof}

\begin{thm}\label{main_descent_stochastic_leader}
Let $f$ satisfy Assumption 1. If $\eta,\lambda$ are fixed so that $\eta \leq (2M(\nu+1))^{-1}$ and $\eta\lambda \leq (2\kappa)^{-1}, \eta\sqrt{\lambda} \leq (\kappa\sqrt{2m})^{-1}$, then $\limsup\limits_{k \rightarrow \infty} \EX f(x_k) - f(x^\ast) \leq \frac{1}{2}\eta \kappa\sigma^2 + \frac{4}{m} \lambda \sqrt{p} \sigma_f$. If $\eta, \lambda$ decrease at the rate $\eta_k = \Theta(\frac{1}{k}), \lambda_k = \Theta(\frac{1}{k})$, then $\EX f(x_k) - f(x^\ast) = O(\frac{1}{k})$.
\end{thm}
\begin{proof}
Interpret the term $4\eta\lambda\sqrt{p}\sigma_f$ as additive noise. Note that if $\eta_k, \lambda_k = \Theta(\frac{1}{k})$, then $\eta\lambda = \Theta(\frac{1}{k^2})$. The proof is then similar to \Cref{main_descent_standard_rates} and follows from Theorems 4.6 and 4.7 of \cite{BCN2018SIAMREV}. 
\end{proof}

\subsection{Proofs from \Cref{sub:nonconvex}}
\begin{thm}
Let $\Omega_i$ be the set of points $(\ux{x}{1},\ldots,\ux{x}{p})$ where $\ux{x}{i}$ is the unique minimizer among $(\ux{x}{1},\ldots,\ux{x}{p})$\footnote{The uniqueness of the minimizer on $\Omega_i$ is only to avoid ambiguities in $\argmin$.}. 
Let $x^\ast = (\ux{w}{1},\ldots,\ux{w}{p}) \in \Omega_i$ be a stationary point of the LGD objective function \eqref{eq:lsgd_obj}. Then $\nabla \ux{f}{i}(\ux{w}{i}) = 0$.
\end{thm}
\begin{proof}
This follows from the fact that on $\Omega_i$, $\frac{\partial L}{\partial \ux{x}{i}} = \nabla \ux{f}{i}(\ux{x}{i})$.
\end{proof}

\begin{lma}\label{leader_descent}
Let $f$ be $M$-Lipschitz-differentiable. Let $\wt{x}_k$ denote the leader at the end of the $k$-th period. If the LGD step size is chosen so that $\ux{\eta}{i} < \frac{2}{M}$, then $f(\wt{x}_k) \leq f(\wt{x}_{k-1})$.
\end{lma}

\begin{proof}
Assume that $\wt{x}_{k-1} = \ux{x}{1}_{k-1}$. Since $\ux{x}{1}$ is the leader during the $k$-th period, the LGD steps for $\ux{x}{1}$ are gradient descent steps. By \Cref{gd_descent}, $\ux{\eta}{1}$ has been chosen so that gradient descent on $f$ is monotonically decreasing, so we know that $f(\ux{x}{1}_k) \leq f(\ux{x}{1}_{k-1})$. Hence $f(\wt{x}_{k}) \leq f(\ux{x}{1}_k) \leq f(\ux{x}{1}_{k-1}) = f(\wt{x}_{k-1})$.
\end{proof}

\begin{thm}\label{lim_stationary}
Assume that $f$ is bounded below and $M$-Lipschitz-differentiable, and that the LGD step sizes are selected so that $\ux{\eta}{i} < \frac{2}{M}$. Then for any choice of communication periods, it holds that for every $i$ such that $\ux{x}{i}$ is the leader infinitely often, $\liminf_k \|\nabla f(\ux{x}{i}_k)\| = 0$.

Note that there necessarily exists an index $i$ such that $\ux{x}{i}$ is the leader infinitely often.
\end{thm}
\begin{proof}
Without loss of generality, we assume it to be $\ux{x}{1}$. Let $\tau(1),\tau(2),\ldots$ denote the periods where $\ux{x}{1}$ is the leader, with $b(k)$ steps in the period $\tau(k)$. By \Cref{leader_descent}, $f(\ux{x}{1}_{\tau(k+1)}) \leq f(\ux{x}{1}_{\tau(k)})$, since the objective value of the leaders is monotonically decreasing. Now, by \Cref{gd_descent}, we have $\sum_{i=0}^{b(k)-1} \|\nabla f(\ux{x}{1}_{\tau(k),i})\|^2 \leq \alpha (f(\ux{x}{1}_{\tau(k),0}) - f(\ux{x}{1}_{\tau(k),b(k)})) = \alpha (f(\ux{x}{1}_{\tau(k)}) - f(\ux{x}{1}_{\tau(k+1)}))$. Since $f$ is bounded below, and the sequence $\{ f(\ux{x}{1}_{\tau(k)})\}$ is monotonically decreasing, we must have $f(\ux{x}{1}_{\tau(k)}) - f(\ux{x}{1}_{\tau(k+1)}) \rightarrow 0$. Therefore, we must have $\|\nabla f(\ux{x}{1}_{\tau(k),i})\| \rightarrow 0$. 
\end{proof}

\subsection{Proofs from \Cref{sub:improve}}

The \emph{cone} with center $d$ and angle $\theta_c$ is defined to be
$$\opn{cone}(d, \theta_c) = \{ x: x^Td \geq 0, \theta(x,d) \leq \theta_c\}.$$
We record the following facts about cones which will be useful.
\begin{prop}\label{scale_cone}
Let $C \subseteq \opn{cone}(d, \theta_c)$. If $y$ is a point such that $sy \in C$ for some $s \geq 0$, then $y \in \opn{cone}(d, \theta_c)$.
\end{prop}
\begin{proof}
This follows immediately from the fact that $\theta(y, d) = \theta(sy, d)$ for all $s \geq 0$.
\end{proof}
\begin{prop}\label{cone_normal}
Let $C = \opn{cone}(d, \theta_c)$ with $\theta_c > 0$. The outward normal vector at the point $x \in \partial C$ is given by $N_x = x - \frac{\|x\|}{\cos(\theta_c)\|d\|}d$. Moreover, if $v$ satisfies $N_x^Tv < 0$, then for sufficiently small positive $\lambda$, $x + \lambda v \in \opn{cone}(d,\theta_c)$.
\end{prop}
\begin{proof}
The first statement follows from the second, by the supporting hyperplane theorem.

Write $\gamma = \cos(\theta_c)$. Let $N_x = x - \frac{\|x\|}{\gamma \|d\|}d$, and let $v$ be a unit vector with $N_x^Tv = x^Tv - \frac{\|x\|}{\gamma \|d\|}d^Tv < 0$. The angle satisfies
$$\cos(\theta(x + \lambda v, d)) = \frac{d^T(x+ \lambda v)}{ \|d\|\|x + \lambda v\|} = \frac{d^Tx + \lambda d^Tv}{ \|d\|\sqrt{\|x\|^2 + \lambda^2 \|v\|^2 + 2\lambda x^Tv }}$$
Differentiating, the numerator $g(\lambda)$ of $\frac{\partial}{\partial \lambda} \cos(\theta(x+\lambda v, d))$ is given by
$$g(\lambda) = \|x\|^2 v^Td - x^Tv x^Td 
+ \lambda \cdot (2 v^Td x^Td + \|v\|^2(\lambda v-x)^Td - \lambda \|v\|^2 v^Td - x^Tv v^Td)$$
Evaluating at $\lambda = 0$ and using $x^Tv - \frac{\|x\|}{\gamma \|d\|} d^Tv < 0$, we obtain
\begin{align*}
g(0) = \|x\|^2 v^Td - x^Tv x^Td &= \|x\|^2 v^Td - x^Tv(\gamma \|x\| \|d\|) \\
&= \|x\|( \|x\|v^Td - \gamma \|d\| x^Tv) > 0.
\end{align*}
Therefore, for small positive $\lambda$, we have $\cos(\theta(x + \lambda v, d)) > \cos(\theta(x,d)) = \gamma$, so $x + \theta v \in \opn{cone}(d, \theta_c)$.
\end{proof}

\begin{prop}\label{imp_small_force}
Let $x$ be any point such that $\theta_x = \theta(d_G(x),d_N(x)) > 0$, and let $E = \{ z: f(z) \leq f(x)\}$. Let $C = \opn{cone}(-x, \theta_x)$, and let $N_x$ be the outward normal $-\nabla f(x) + \frac{\|\nabla f(x)\|}{\cos(\theta_x)\|x\|}x$ of the cone $C$ at the point $-\nabla f(x)$. Then
\begin{equation}\label{eq:angle_improv_char}\bigcup_{\lambda > 0} I_\theta(x,\lambda) \supseteq E \cap \{z: N_x^Tz <N_x^Tx\}\end{equation}
and consequently, $\lim_{\lambda \rightarrow 0} \opn{Vol}(I_\theta(x,\lambda)) \geq \frac{1}{2} \opn{Vol}(E)$.
\end{prop}
\begin{proof}
First, note that if $\lambda_2 \leq \lambda_1$, then for all $z$ with $-\nabla f(x) + \lambda_1 z \in C$, we also have $-\nabla f(x) + \lambda_2 z \in C$ by the convexity of $C$. Therefore $I_\theta(x, \lambda_2) \supseteq I_\theta(x,\lambda_1)$, so $\lim_{\lambda \rightarrow 0} \opn{Vol}(I_\theta(x,\lambda))$ exists. We first prove the second statement. For any normal vector $h$ and $\beta > 0$, $\opn{Vol}(E \cap \{ z: h^Tz < \beta\}) \geq \frac{1}{2}\opn{Vol}(E)$, since the center $0 \in \{z: h^Tz < \beta\}$. The result follows because $N_x^Tx > 0$.

To prove (\ref{eq:angle_improv_char}), observe that $z \in I_\theta(x,\lambda)$ if equivalent to $-\nabla f(x) + \lambda (z - x) \in \opn{cone}(-x, \theta_c)$. By \Cref{cone_normal}, there exists $\lambda > 0$ with $-\nabla f(x) + \lambda(z-x) \in \opn{cone}(-x,\theta_c)$ if $N_x^T(z-x) < 0$. Hence, it follows that every point in $E \cap \{z: N^Tz < N^Tx\}$ is contained in $I_\theta(x,\lambda)$ for some $\lambda > 0$.
\end{proof}

\begin{lma}
There exists a direction $x$ such that $\cos(\theta(d_G(x), d_N(x))) = 2(\sqrt{\kappa}+\sqrt{\kappa^{-1}})^{-1}$. Thus, for all $r \geq 2$, there exists a direction $x$ with $\cos(\theta(d_G(x), d_N(x))) \leq \frac{r}{\sqrt{\kappa}}$.
\end{lma}
\begin{proof}
Take $x = \sqrt{\frac{\alpha_n}{\alpha_1+\alpha_n}} e_1 + \sqrt{\frac{\alpha_1}{\alpha_1+\alpha_n}}e_n$. It is easy to verify that $\cos(\theta(d_G, d_N)) = 2(\sqrt{\kappa}+\sqrt{\kappa^{-1}})^{-1}$.
\end{proof}

\begin{prop}\label{orth_big_angle}
For any $x$, let $\theta_x = \theta(d_G(x), d_N(x))$. We have
$$\max \{ \|z\|^2: f(z) \leq f(x), z^Tx = 0\} \leq \kappa \cos(\theta_x) \|x\|^2$$
\end{prop}
\begin{proof}
Form the maximization problem
$$
\left\{ \begin{array}{cl} \max\limits_z & z^Tz \\
& z^TAz \leq x^TAx \\
& z^Tx = 0
\end{array} \right.$$
The KKT conditions for this problem imply that the solution satisfies $z - \mu_1 Az - \mu_2 x = 0$, for Lagrange multipliers $\mu_1 \geq 0, \mu_2$. Since $z^Tx = 0$, we obtain $z^Tz = \mu_1 z^TAz$, and thus $\frac{1}{M} \leq \mu_1 \leq \frac{1}{m}$. Since $f(z) \leq f(x)$, we find that $z^Tz \leq \frac{1}{m} x^TAx$. Using $\cos(\theta_x) = \frac{x^TAx}{\|x\|\|Ax\|}$, we obtain
$$z^Tz \leq \frac{1}{m} \cos(\theta_x) \|x\|\|Ax\| \leq \kappa \cos(\theta_x) \|x\|^2.$$
\end{proof}

\begin{thm}\label{imp_big_angle}
Let $R_\kappa = \{r: \frac{r}{\sqrt{\kappa}} + \frac{r^{3/2}}{\kappa^{1/4}} \leq 1\}$. Let $x \in S_r$ for $r \in R_\kappa$, and let $E = \{y: f(y) \leq f(x)\}$, $E_2 = \{z \in E: z^Tx \leq 0 \}$, $\theta_x = \theta(d_G(x), d_N(x))$. Then for all $z \in E_2$ and any $\lambda \geq 0$, the LGD direction $d_z = -(\nabla f(x) + \lambda(x-z))$ satisfies $\theta(d_z, d_N(x)) \leq \theta_x$. Thus, $E_2 \subseteq I_\theta(x, \lambda)$, and therefore $\opn{Vol}(I_\theta(x,\lambda)) \geq \opn{Vol}(E_2) = \frac{1}{2}\opn{Vol}(E)$.
\end{thm}
\begin{proof}
Define $D_2 = \{z - x : z \in E_2\}$\footnote{Note the sign change from $x - z$ to $z - x$ here.}. The set of possible LGD directions with $z \in E_2$ is given by $D_3 = \{ -\nabla f(x) + \lambda \delta: \delta \in D_2, \lambda \geq 0 \}$. 
Since $d_N(x) = -x$, our desired result is equivalent to $D_3 \subseteq \opn{cone}(-x, \theta_x)$.

Define the subset $D_2' = \{z - x: z \in E_2, x^Tz = 0\}$. We claim that it suffices to prove that $D_2' \subseteq \opn{cone}(-x, \theta_x)$. To see this, consider any $\lambda \delta$ for $\lambda \geq 0$ and $\delta \in D_2$. We have $x^T(\lambda \delta) = \lambda x^T(z - x) \leq -\lambda x^Tx < 0$, so there exists a scalar $s$ with $x^T(s\lambda \delta) = -x^Tx$, whence $s\lambda \delta \in D_2' \subseteq \opn{cone}(-x, \theta_x)$. By \Cref{scale_cone}, $\lambda \delta \in \opn{cone}(-x, \theta_x)$. Since $-\nabla f(x) \in \opn{cone}(-x, \theta_x)$, convexity implies that $-\nabla f(x) + \lambda \delta \in \opn{cone}(-x, \theta_x)$. Thus, $D_2' \subseteq \opn{cone}(-x, \theta_x)$ implies that $D_3 \subseteq \opn{cone}(-x, \theta_x)$.

To complete the proof, let $\delta = z - x \in D_2'$ and observe that $\cos(\theta(\delta, d_N(x))) = \frac{x^T(x - z)}{\|x\|\|x - z\|}$. By \Cref{orth_big_angle} and the definition of $S_r$,
$$\max\{ \|z\|: z \in E_2, z^Tx = 0\} \leq \sqrt{\kappa} \sqrt{\cos(\theta_x)} \|x\| = \sqrt{r}\kappa^{1/4}\|x\|$$
We compute that
\begin{align*}
x^T(x-z) - \frac{r}{\sqrt{\kappa}}\|x\|\|x - z\| &\geq \|x\|^2 - \frac{r}{\sqrt{\kappa}}(\|x\|^2 + \|x\|\|z\|) \\
&\geq \|x\|^2 - \frac{r}{\sqrt{\kappa}}\|x\|^2 - \frac{r}{\sqrt{\kappa}}\|x\|(\sqrt{r}\kappa^{1/4} \|x\|) \\
&\geq \left( 1 - \frac{r}{\sqrt{\kappa}} - \frac{r^{3/2}}{\kappa^{1/4}} \right) \|x\|^2 \geq 0
\end{align*}
By the definition of $R_\kappa$, this is non-negative, and thus $\theta(\delta, d_N(x)) \leq \theta_x$. This completes the proof.
\end{proof}

\section{Low-Rank Matrix Completion Experiments}

Low-rank matrix completion problem is an example of a non-convex learning problem whose landscape exhibits numerous symmetries. We consider the positive semi-definite case, where the objective is to find a low-rank matrix minimizing
$$\min_X \left\{ F(X) = \frac{1}{4}\|M - XX^T\|_F^2: X \in \RR^{d \times r}\right\}$$

It is routine to calculate that $\nabla F(X) = (XX^T - M)X$. The EAGD and LGD updates for $X$ can be expressed as
$$X_+ = (1 - \eta \lambda) X + \eta\lambda Z - \eta \nabla F(X).$$
For EAGD, $Z = \wt{X}$, and $\wt{X}$ is updated by
$$\wt{X}_+ = (1 - p\eta\lambda)\wt{X} + p\eta\lambda\left( \frac{1}{p} \sum_{i=1}^p X^i \right).$$
For LGD, $Z = \argmin \{F(X^1),\ldots,F(X^p)\}$, and is updated at the beginning of every communication period $\tau$.

The parameters were set to:
$$\eta = \texttt{5e-4}, \lambda = \frac{1}{5}, p = 8, \tau = 1$$
The learning rate $\eta = \texttt{5e-4}$ was selected from a set $\{\texttt{1e-1}, \texttt{5e-2}, \texttt{1e-3},\ldots\}$ by evaluating on a sample problem until a value was found for which both methods exhibited monotonic decrease.

The dimension was $d = 1000$, and the ranks $r \in \{1,10,50,100\}$ were tested. For each rank, there were 10 random trials performed. In each trial, $M$ and starting points $\{X^i_0\}$ are sampled. $M$ is generated by sampling $U \in \RR^{d \times r}$ with i.i.d entries from $N(0,1)$, and taking $M = UU^T$. Initial points for each worker node $X^i$ were also sampled from $N(0,1)$. The same starting points were used for EAGD and LGD.

Code for this experiment is available at \url{https://github.com/wgao-res/lsgd_matrix_completion}.

\section{Experimental Setup}
\subsection{Data preprocessing}
For CIFAR-$10$ experiments we use the original images of size $3 \times 32 \times 32$. We then normalize each image by mean $(0.4914,0.4822,0.4465)$ and standard deviation $(0.2023,0.1994,0.2010)$. We also augment the training data by horizontal flips with a probability of $0.5$.

For CNN$7$ and ResNet$20$, we extract random crops of size $3 \times 28 \times 28$ and present these to the network in batches of size $128$. The test loss and test error are only computed from the center patch $(3 \times 28 \times 28)$ of test images. 

For VGG$16$ we pad the images to $3 \times 40 \times 40$, extract random crops of size $3 \times 32 \times 32$ and present these to the network in batches of size $128$. The test loss and test error are computed from the test images.

For ImageNet experiments we normalize each image by mean $(0.485, 0.456, 0.406)$ and standard deviation $(0.229, 0.224, 0.225)$. We sample the training data in the same way as~\cite{InceptionNet}. For each image, a crop of random size (chosen from $8\%$ to $100\%$ evenly) of the original size and a random aspect ratio (chosen from $3/4$ to $4/3$ evenly) of the original aspect ratio is made. Then we resize the crop to $3 \times 224 \times 224$. We also augment the training data by horizontal flips with a probability of $0.5$.
%\footnote{This is accomplished by using transformer from PyTorch: https://pytorch.org/docs/stable/torchvision/transforms.html}.
Finally we present these to the network in the batches of size $32$. The test images are resized so that the smaller edge of each image is $256$. The test loss and test error are only computed from the center patch $(3 \times 224 \times 224)$ of test images. 

\subsection{Data prefetching}
We use the dataloader and distributed data sampler\footnote{https://pytorch.org/docs/stable/data.html} from PyTorch. Each worker loads a subset of the original data set that is exclusive to that worker for every epoch. If the size of data set is not divisible by the batch size, the last incomplete batch will be dropped.

\subsection{Hyperparameters}
In Table \ref{tab:cifar_lr} we summarize the learning rates and other hyperparameters explored for each method in the experiments on CIFAR-$10$. The setting of $\beta$ for EASGD was obtained from the original paper (its authors use this setting for all their experiments). We do learning rate drop at $500$ seconds by a factor of $0.1$ for all the methods.
\begin{table}[H]
\caption{Hyperparameters: CNN$7$/VGG$16$/ResNet$20$ experiment on CIFAR-$10$.}
\label{tab:cifar_lr}
\centering
\begin{tabular}{|c||c|c|c|c| }
 \hline
 Name & Learning Rates &Comm. Period &Batch Sizes &\\
 \hline
 DataParallel & $\{0.1, 0.01, 0.001, 0.0001\}$ &$\tau=1$ &$\{4, 16, 64, 128\}$ &\\
 \hline
 LARS & $\{100, 10, 1.0, 0.1, 0.01\}$ &$\tau=1$ &$\{4, 16, 64, 128\}$ &\\
 \hline
 EASGD    & $\{0.1, 0.01, 0.001\}$ &$\tau = \{1,4,16,64\}$ &128 &$\beta = 0.43$ \\
 \hline
 LSGD     & $\{0.1, 0.01, 0.001\}$ &\makecell{$\tau = \{1,4,16,64\}$, \\ $\tau_G = \{16, 64\}$} &128 &\makecell{$\lambda = \lambda_G =$ \\ $\{0.1, 0.05\}$, \\ $\gamma = \gamma_G =$ \\ $\{0.1, 0.01\}$} \\
 \hline
\end{tabular}
\end{table}

In Table \ref{tab:imagenet_lr} we summarize the learning rates and other hyperparameters explored for each method in experiments on ImageNet. We do learning rate drop for every $30$ epochs by a factor of $0.1$ for all the methods.
\begin{table}[H]
\caption{Hyperparameters: ResNet$50$ experiment on ImageNet.}
\label{tab:imagenet_lr}
\centering
\begin{tabular}{|c||c|c|c|c| }
 \hline
 Name & Learning Rates &Comm. Period & Batch Sizes &\\
 \hline
 DataParallel &$0.1$ &$\tau=1$ &$32$ &\\
 \hline
 LARS  &$\{10, 1.0, 0.1\}$ &$\tau=1$ &$32$ &\\
 \hline
 EASGD &$0.1$ &$\tau = \{1,4,16,64\}$ &$32$ &$\beta = 0.43$ \\
 \hline
 LSGD  &$0.1$ &$\tau = \{1,4,16,64\}$, &$32$ & $\lambda = \lambda_G = \{0.2, 0.1\}$, \\
 & &$\tau_G = \{4, 16, 64\}$ & &$\gamma=\gamma_G=\{0.1, 0.0\}$ \\
 \hline
\end{tabular}
\end{table}

\subsection{Implementation Details}
\label{sec:package}
To take advantage of both the efficiency of collective communication and the flexibility of peer-to-peer communication, we incorporate two backends, namely NCCL and GLOO\footnote{https://github.com/facebookincubator/gloo}, for GPU processors and CPU processors, respectively. 
%We developed a software framework based on PyTorch, that supports both local and distributed computation in neural networks. 

The global and local servers (running on CPU processors) control the training process and the workers (running on GPU processors) perform the actual computations. For each iteration each worker has only one of the following two choices:
\begin{enumerate}
\item Local Training: Each worker is trained with one batch of the training data;
\item Distributed Training: Each worker communicates with other workers and updates its parameters based on the pre-defined distributed training method.
\end{enumerate}
To minimize the cost of communication over Ethernet, the global server is running on the first GPU node instead of a separate machine. Also, for a fair comparison, the center variable is being maintained and updated by the first GPU node as well\footnote{In the original implementation of~\cite{EASGD} and~\cite{DOWNPOUR}, an individual parameter server is used for updating the center variable based on the peer-to-peer communication scheme. However, there is no need to use an individual parameter server under collective communication scheme as it will only induce extra communication cost.}. 

\begin{figure}[H]
\centering
\includegraphics[width=.7\linewidth]{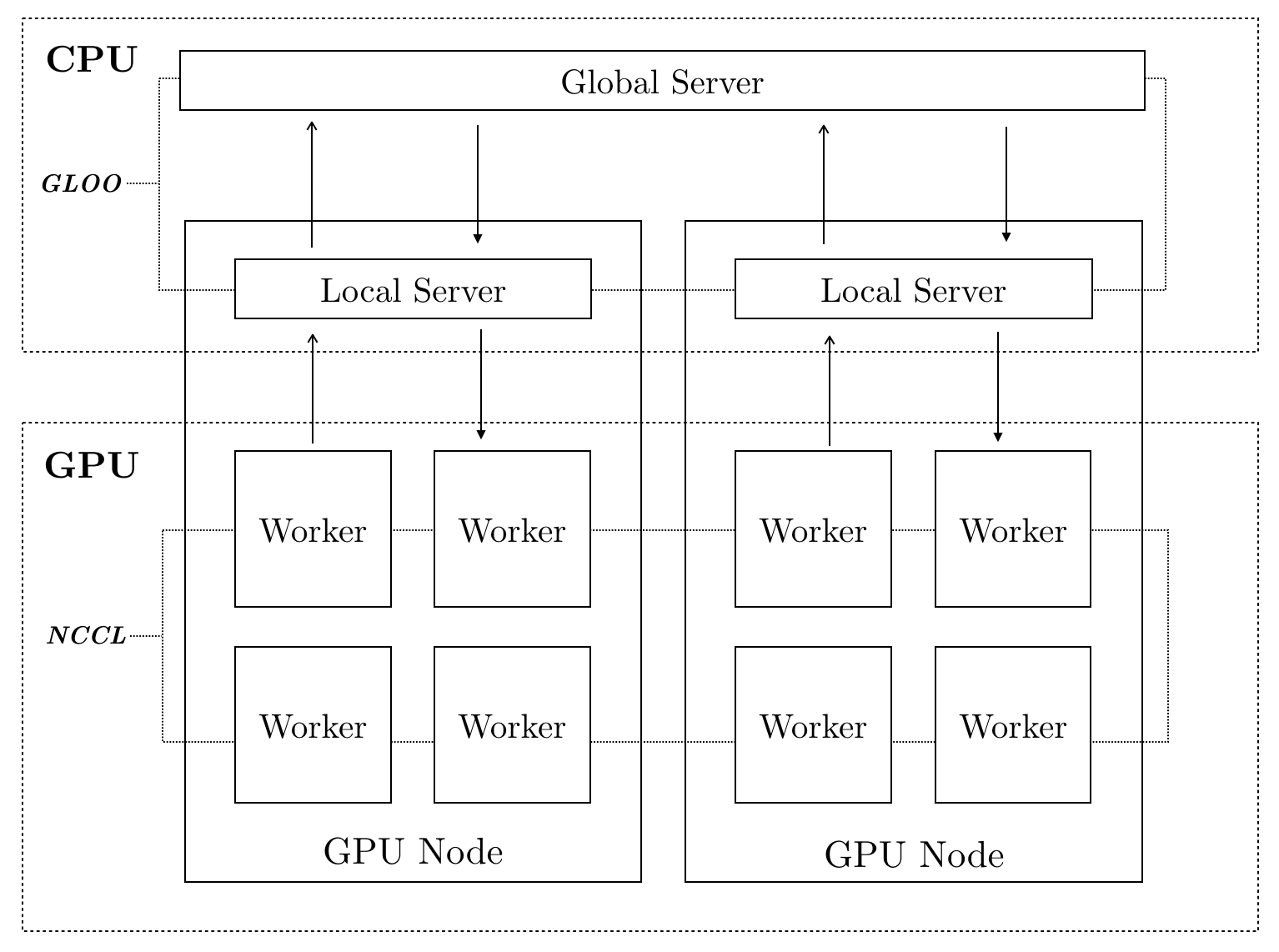}
\caption{At the beginning of each iteration, the local worker sends out a request to its local server and then the local server passes on the worker's request to the global server. The global server checks the current status and replies to the local server. The local server passes on the global server's message to the worker. Finally, depending on the message from the global server, the worker will choose to follow the local training or distributed training scheme.}
\end{figure}

\section{More results from Section \ref{sec:ExpA}}
\label{sec:ExpC}
\vspace{-0.35in}
\begin{figure}[H]
\centering
\includegraphics[width=0.48\linewidth]{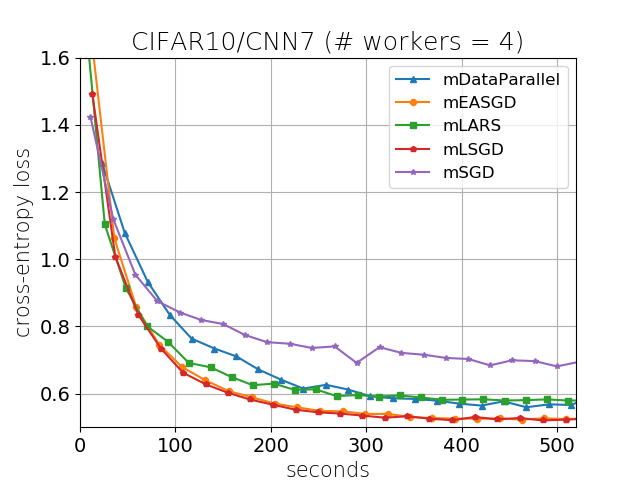}
\includegraphics[width=0.48\linewidth]{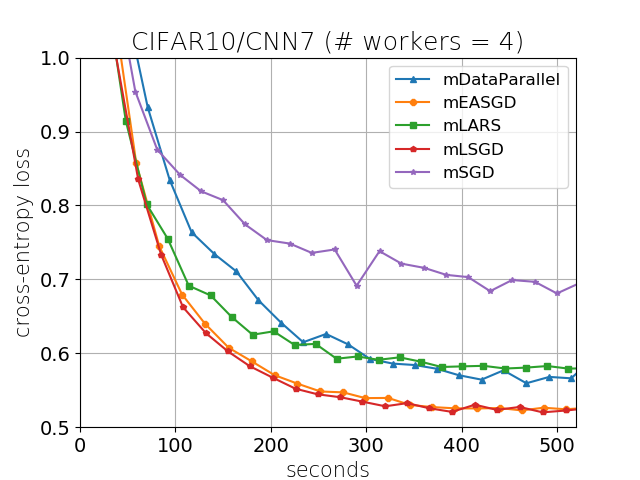}
\includegraphics[width=0.48\linewidth]{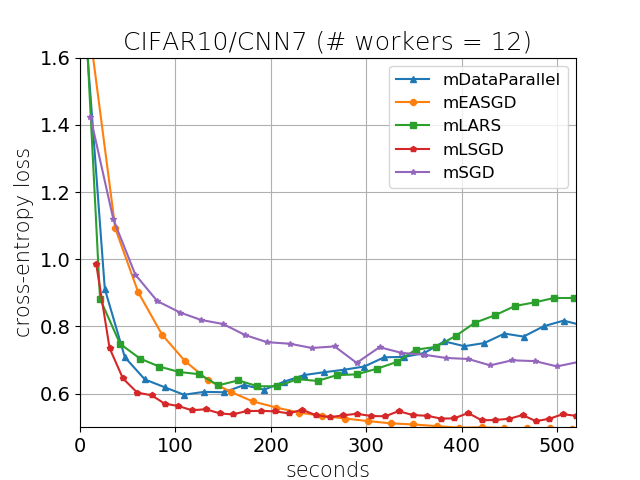}
\includegraphics[width=0.48\linewidth]{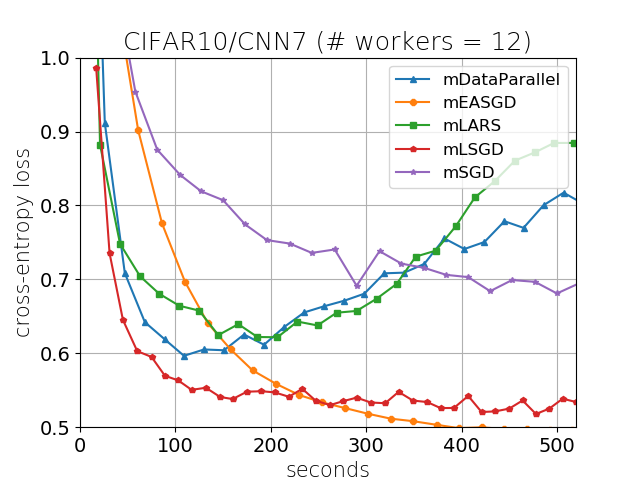}
\caption{CNN$7$ on CIFAR-$10$. Test loss for the center variable versus wall-clock time (original plot on the left and zoomed on the right).}
\label{fig:CNN7testloss}
\end{figure}

\vspace{-0.35in}
\begin{figure}[H]
\centering
\includegraphics[width=0.48\linewidth]{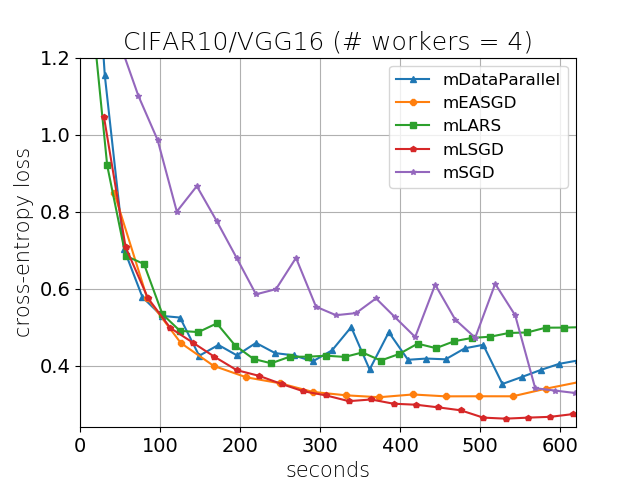}
\includegraphics[width=0.48\linewidth]{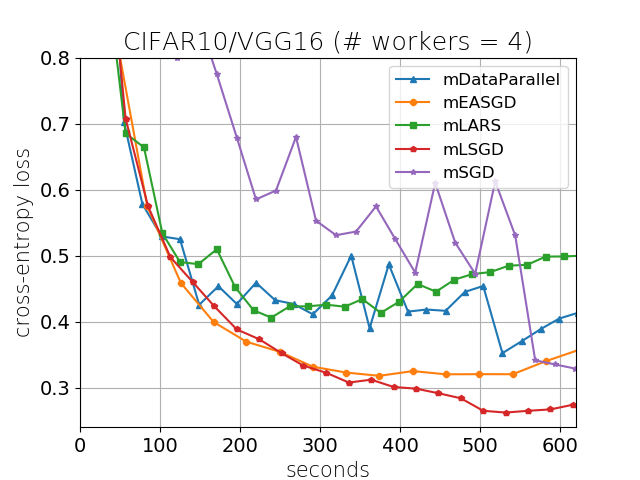}
\includegraphics[width=0.48\linewidth]{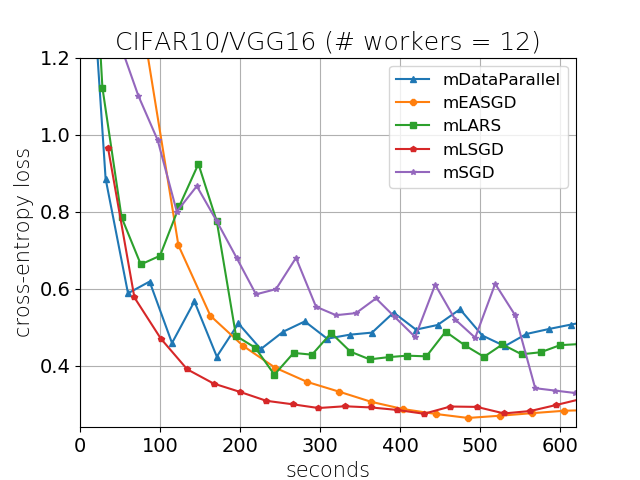}
\includegraphics[width=0.48\linewidth]{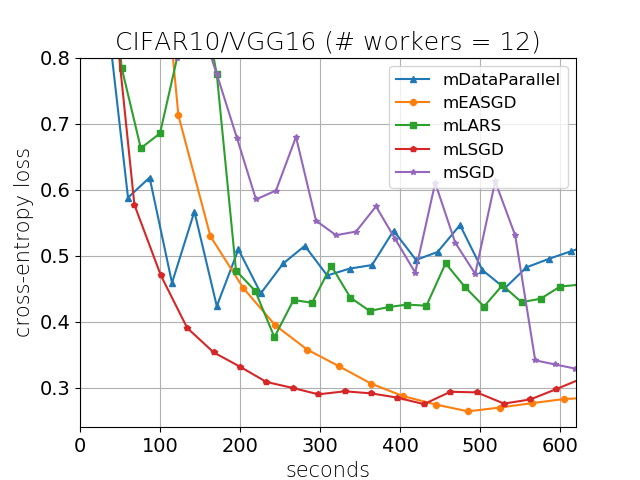}
\caption{VGG$16$ on CIFAR-$10$. Test loss for the center variable versus wall-clock time (original plot on the left and zoomed on the right).}
\label{fig:VGG16testloss}
\end{figure}

\vspace{-0.35in}
\begin{figure}[H]
\centering
\includegraphics[width=0.48\linewidth]{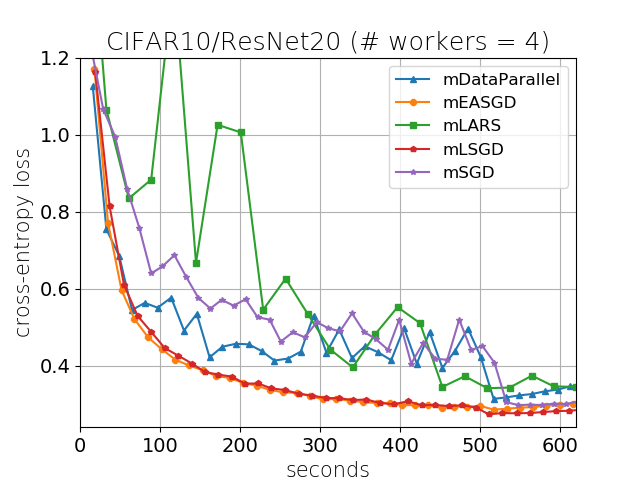}
\includegraphics[width=0.48\linewidth]{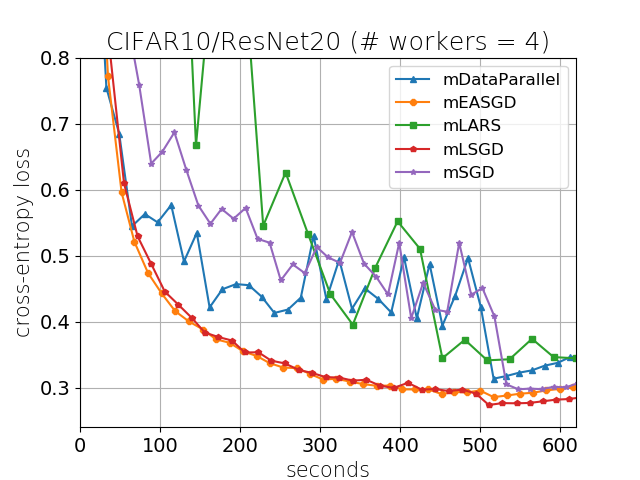}
\includegraphics[width=0.48\linewidth]{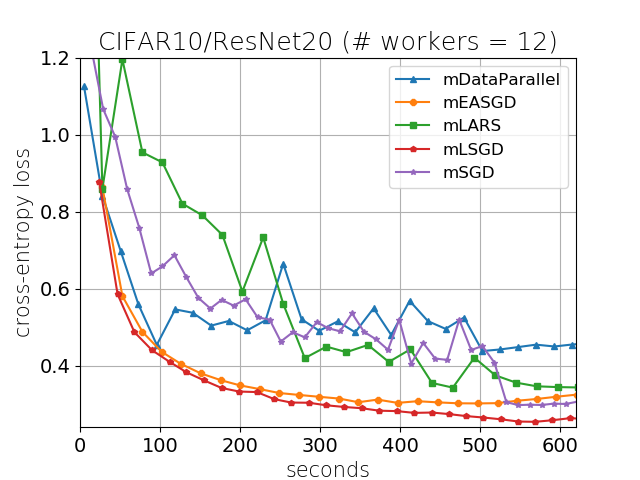}
\includegraphics[width=0.48\linewidth]{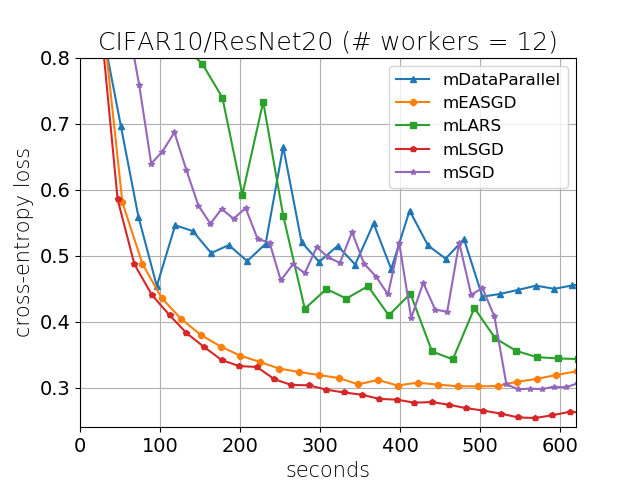}
\caption{ResNet$20$ on CIFAR-$10$. Test loss for the center variable versus wall-clock time (original plot on the left and zoomed on the right).}
\label{fig:ResNet20testloss}
\end{figure}

\vspace{-0.35in}
\begin{figure}[H]
\centering
\includegraphics[width=0.48\linewidth]{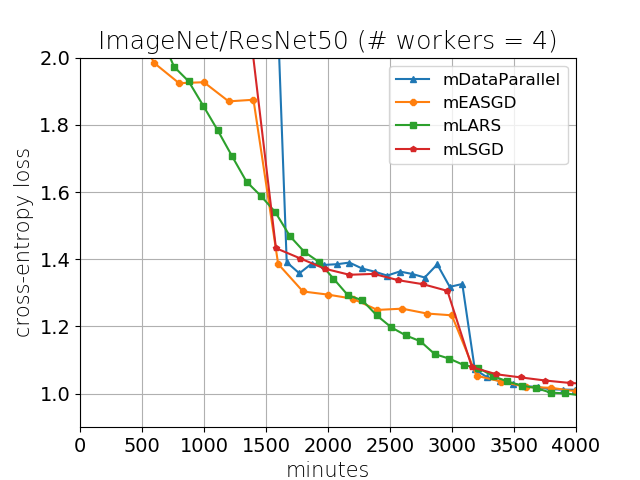}
\includegraphics[width=0.48\linewidth]{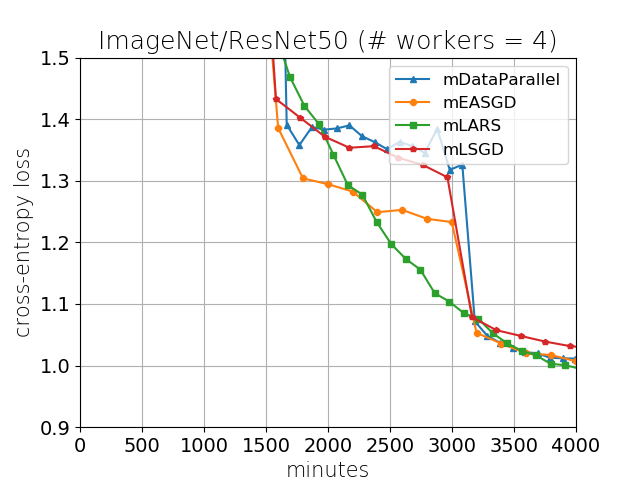}
\includegraphics[width=0.48\linewidth]{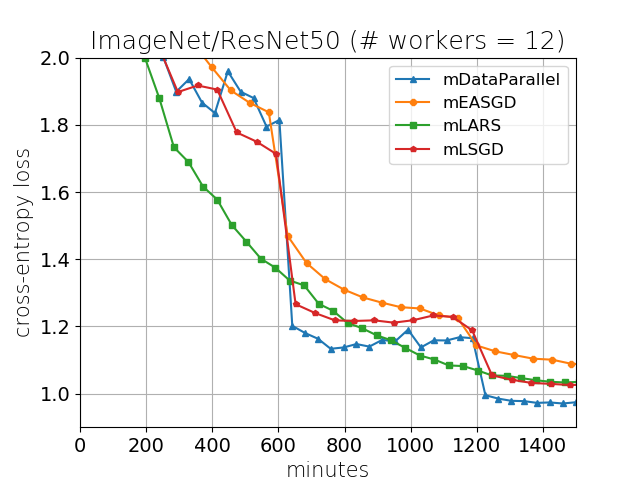}
\includegraphics[width=0.48\linewidth]{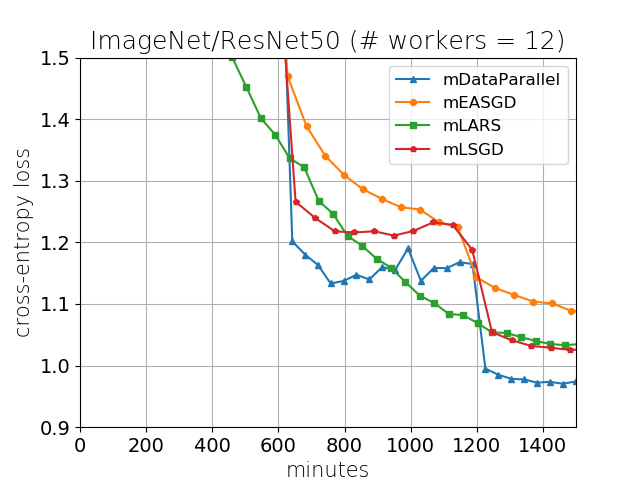}
\caption{ResNet$50$ on ImageNet. Test loss for the center variable versus wall-clock time (original plot on the left and zoomed on the right).}
\label{fig:ResNet50testloss}
\end{figure}
\end{document}